%% file: main.tex
\documentclass{article}

\usepackage[preprint]{neurips_2026}

\usepackage[utf8]{inputenc} % allow utf-8 input
\usepackage[T1]{fontenc}    % use 8-bit T1 fonts
\usepackage{hyperref}       % hyperlinks
\usepackage{url}            % simple URL typesetting
\usepackage{booktabs}       % professional-quality tables
\usepackage{amsfonts}       % blackboard math symbols
\usepackage{nicefrac}       % compact symbols for 1/2, etc.
\usepackage{microtype}      % microtypography
\usepackage{xcolor}         % colors

\usepackage{algorithm}
\usepackage{graphicx}
\usepackage{amsmath}
\usepackage{amssymb}
\usepackage{amsthm}
\usepackage{multirow}
\usepackage{float}
\usepackage{makecell}
\usepackage{tabularx}
\usepackage{enumitem}

\newtheorem{proposition}{Proposition}[section]

\DeclareMathOperator*{\argmax}{argmax}
\usepackage{algorithmic}
\definecolor{slatepurple}{HTML}{053BC1}    
\definecolor{deepred}{HTML}{DE0000}     

\newcommand{\systemname}{AlphaTransit}  
\title{\systemname: Learning to Design City-scale \\Transit Routes}

\author{%
  Bibek Poudel\textsuperscript{1}
  \quad
  Sai Swaminathan\textsuperscript{1}
  \quad
  Weizi Li\textsuperscript{2}
  \\
  \textsuperscript{1}Department of EECS, University of Tennessee, Knoxville, TN, USA
  \\
  \textsuperscript{2}Department of CSE, University of California, Riverside, CA, USA
  \\
  Correspondence: \texttt{bpoudel3@vols.utk.edu}
}

\begin{document}

\maketitle

%%%%%%%%%%%%%%%%%%%%%%%%%%%%%%%%%%%%%%%%%%%%%%%%%%%%%%%%%%%%

\input{sections/abstract}

\input{sections/intro}

\input{sections/related}
\input{sections/method}

\input{sections/experiment}

\input{sections/conclusion}

\newpage
\bibliography{main}{}
\bibliographystyle{plain}

%%%%%%%%%%%%%%%%%%%%%%%%%%%%%%%%%%%%%%%%%%%%%%%%%%%%%%%%%%%%

\appendix
\newpage
% \clearpage
\setcounter{page}{1}
\input{appendix_sections/search_space}
\input{appendix_sections/frequency_assignment}
\input{appendix_sections/state_representation}
\input{appendix_sections/policy_network}
\input{appendix_sections/training_procedures_hyperparameters}
\input{appendix_sections/real_world_networks_data}
\input{appendix_sections/results}
\input{appendix_sections/broader_impacts}
\newpage
% \input{sections/checklist}

%%%%%%%%%%%%%%%%%%%%%%%%%%%%%%%%%%%%%%%%%%%%%%%%%%%%%%%%%%%%

\end{document}

%% file: sections/abstract.tex
\begin{figure*}[h!]
    \centering
    \vspace{-24pt}
    \includegraphics[width=0.98\linewidth]{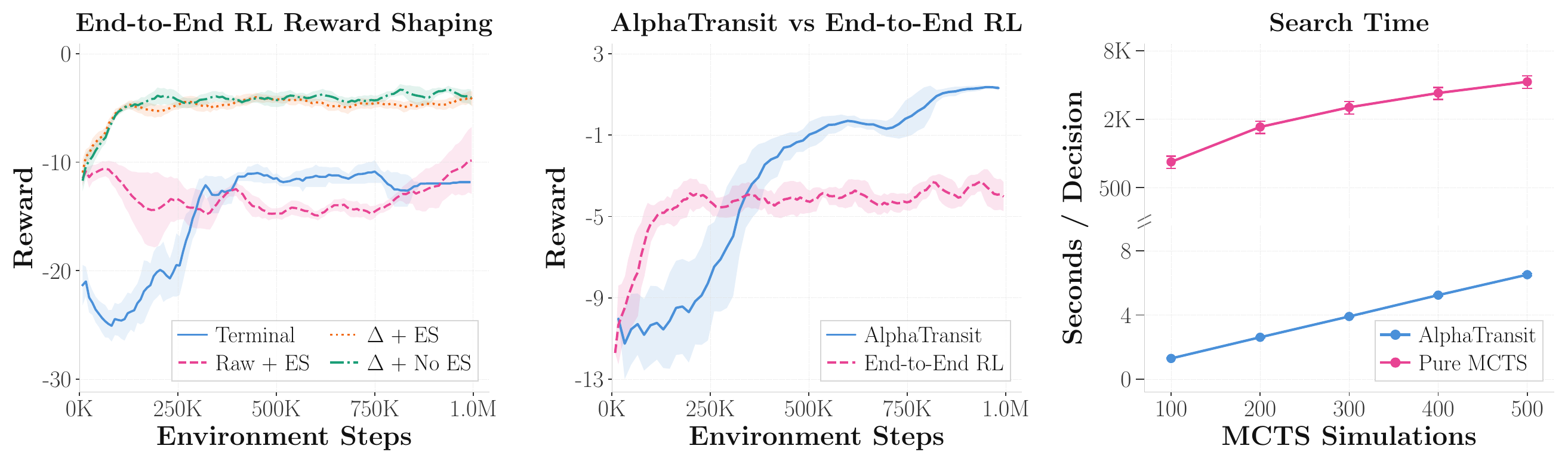}
    \vspace{-6pt}
    \caption{Learning dynamics and search cost under mixed demand ($\alpha=0.3$). \textbf{LEFT}: Curves show averages over two training seeds per reward mode. Reward shaping is critical for End-to-End RL. Early Stopping (ES) penalizes routes that terminate before $L_{\max}$, while Delta-coverage ($\Delta$) rewards only newly covered demand. The $\Delta$ variants are strongest; $\Delta$ + No ES is best. \textbf{MIDDLE}: MCTS search improves sample efficiency under the same environment-step budget. At search depth $N_{\mathrm{iter}}=500$, \systemname{} surpasses End-to-End RL around $3\times10^5$ steps and finishes with a final smoothed reward $1.31$ versus $-4.02$. \textbf{RIGHT}: Learned policy and value estimates make MCTS search practical across search depths $N_{\mathrm{iter}}\in[100,500]$, when benchmarked on the Bloomington network with a single CPU worker. \systemname{} stays within seconds per decision, whereas Pure MCTS requires hundreds to thousands of seconds per decision.}
    \vspace{-4pt}
    \label{fig:learning_overview_0_3}
\end{figure*}

\begin{abstract}
Designing a transit network requires many sequential route extension decisions, but their quality is often visible only after the full network is assembled. This delayed-feedback challenge lies at the heart of the Transit Route Network Design Problem (TRNDP), where route interactions can be deceptive: an extension that appears useful locally can create transfer bottlenecks, produce redundant overlap, or reduce overall throughput. To guide route construction under delayed simulator feedback, we introduce \systemname{}, a search-based planning framework for city-scale bus network design. \systemname{} couples Monte Carlo Tree Search (MCTS) with a neural policy-value network: the policy proposes route extensions, the value estimates downstream design quality, and search uses these predictions to refine each decision. This provides decision-time lookahead during route construction without running simulator rollouts inside the search tree. We evaluate \systemname{} on a new Bloomington TRNDP benchmark with realistic road topology and census-derived demand, under mixed and full transit demand settings. In the Bloomington network, \systemname{} attains the highest service rate in both demand settings, reaching $54.6\%$ and $82.1\%$, respectively. Relative to reinforcement learning without search, these correspond to $9.9\%$ and $11.4\%$ service rate gains; relative to MCTS without learned guidance, they correspond to $2.5\%$ and $11.2\%$ gains. These results suggest that coupling learned guidance with MCTS is more effective than using either approach alone for transit network design. Our code and data are publicly available in \url{https://github.com/poudel-bibek/AlphaTransit}.

\end{abstract}

%% file: sections/intro.tex
\section{Introduction}
\label{sec:intro}

Learning to act under sparse rewards is a central challenge in sequential decision making~\cite{sutton2018reinforcement}. In such settings, combining policy priors and value estimates with Monte Carlo Tree Search (MCTS) has notably achieved superhuman performance in games such as Go, chess, and shogi~\cite{silver2016mastering,silver2017mastering,silver2018general}. This approach has also succeeded with unknown dynamics by planning with a learned model~\cite{schrittwieser2020mastering}, and has yielded breakthroughs in algorithmic discovery for matrix multiplication~\cite{fawzi2022discovering} and low-level sorting routines~\cite{mankowitz2023faster}. Together, these results suggest that learned priors combined with explicit lookahead extract a stronger training signal from sparse reward environments than policy learning alone.

We study this principle in \emph{Transit Route Network Design Problem} (TRNDP), a combinatorial optimization problem in which a planner designs a set of possibly overlapping routes on a shared street network to serve many-to-many passenger flows under operational constraints such as fleet size and cost~\cite{schmidt2024planning,fan2004optimal}. Unlike classical routing problems such as the travelling salesman problem~\cite{lawler1985traveling} or the vehicle routing problem~\cite{toth2014vehicle}, TRNDP jointly designs a route network whose components interact through transfers, congestion, and shared infrastructure~\cite{ceder2016public}. As a result, changing one route can reassign passengers across the network and alter the value of distant route segments. The reward signal is therefore delayed and nonlocal, since the effect of each route extension is observed only after the full network is assembled and simulated. TRNDP is NP-hard~\cite{kepaptsoglou2009transit,owais2026transit}; the Bloomington setting studied in this work admits $\approx 10^{82}$ candidate route sets (Appendix~\ref{app:searchspace}).
The optimization landscape is also deceptive: a locally attractive route extension can create transfer bottlenecks, redundant overlap, and lower global throughput. In addition, the objective is inherently multi criteria, balancing passenger facing metrics such as coverage, waiting time, and journey time against operator side constraints such as fleet size, route overlap, and vehicle utilization.

\begin{figure*}[t!]
    \centering
    \includegraphics[width=0.98\linewidth]{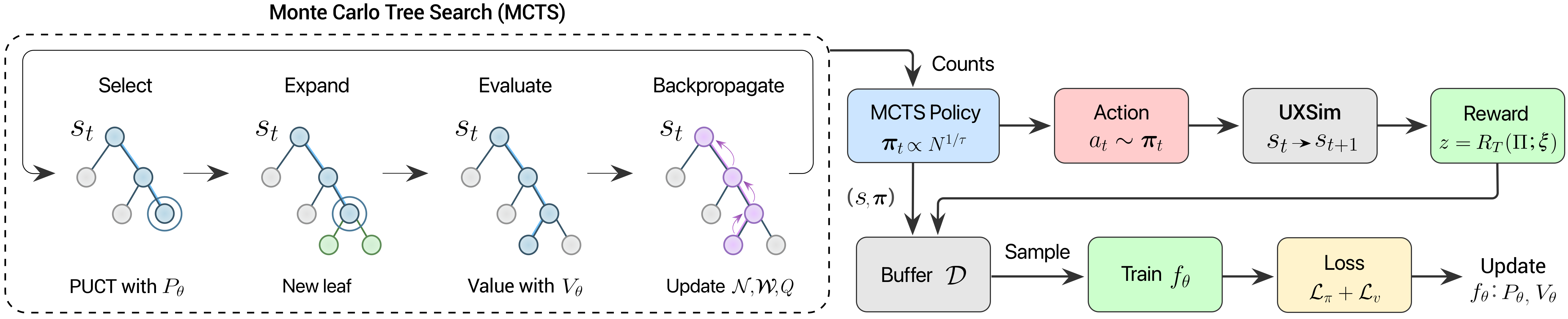}
    \vspace{-4pt}
    \caption{\systemname{} overview. At each route-construction state $s_t$, MCTS uses the policy-value network $f_\theta=(P_\theta,V_\theta)$ to perform selection, expansion, evaluation, and backpropagation. The visit count statistics define a MCTS policy $\pi_t$, from which the next action is sampled to advance the state. After the full route set is completed, UXsim evaluates the design and returns a terminal reward $z$. Tuples $(s_t,\pi_t,z)$ are stored in buffer $\mathcal{D}$ and sampled to train $f_\theta$ with policy and value losses.}
    \label{fig:alphatransit}
    \vspace{-10pt}
\end{figure*}

Prior TRNDP methods handle this difficulty by simplifying the evaluation model, iteratively modifying candidate route sets, or learning route construction policies. Exact and decomposition-based methods follow the first path, making optimization tractable by using analytical passenger routing and travel time objectives rather than simulation~\cite{bertsimas2021data,vermeir2021exact}. Metaheuristics follow the second path, generating and modifying candidate route sets with problem specific heuristic or evolutionary operators~\cite{mumford2013new,nikolic2013transit}. More recently, learning based methods reduce manual design by learning policies that construct routes directly or heuristics that guide search~\cite{yoo2023reinforcement,holliday2025learning}. Across these approaches, local design choices remain hard to assess because their value depends on the completed route set.

This limitation motivates a route construction procedure that improves local decisions through lookahead over partial networks, while avoiding expensive simulation at every node. To address this, we introduce \systemname{}, a search-guided framework that couples MCTS with a neural policy-value network, shown in Fig.~\ref{fig:alphatransit}. For each partial route, the policy provides priors over feasible extensions, and the value estimates downstream design quality. MCTS combines these predictions into a visit-count policy, from which the next extension is sampled to advance the design. During search, leaf states are evaluated by the value head rather than simulator rollouts. The simulator returns a terminal reward only after a complete route set is constructed; the resulting samples are used to update the network to match search-improved policies and predict final rewards. Our contributions are:
\vspace{-4pt}
\begin{itemize}[leftmargin=12pt]
    \setlength{\itemsep}{0pt}
    \setlength{\parskip}{0pt}
    \item We introduce \systemname{} and show that learned lookahead is effective for simulator-defined bus TRNDP, achieving the highest service rate among evaluated methods in both demand regimes.
    \item We present a Bloomington TRNDP benchmark with a topologically accurate road network, census-derived origin-destination demand, and an existing human-designed bus transit reference.
    \item We isolate the role of learned lookahead by comparing \systemname{} against End-to-End Reinforcement Learning (RL) without decision-time search and Pure MCTS without learned priors or value estimates, alongside heuristic, metaheuristic, neural-evolutionary, and real-world baselines.
\end{itemize}

We evaluate \systemname{} on the Bloomington benchmark under mixed and full transit demand regimes. \systemname{} reaches mean service rates of $54.64\%$ and $82.08\%$, respectively, yielding relative service-rate gains of $9.9\%$ and $11.4\%$ over End-to-End RL, and $2.5\%$ and $11.2\%$ over Pure MCTS. Under mixed-demand, \systemname{} also attains the highest bus utilization and second-highest route efficiency. Under full transit demand, it obtains the lowest wait time, highest route efficiency, and highest bus utilization. Together, these comparisons indicate that coupling learned priors with lookahead search makes sparse terminal feedback more useful for route design.

%% file: sections/related.tex
\section{Related Work}
\label{sec:related}
TRNDP has been studied for over four decades, spanning exact, heuristic, and learning-based methods~\cite{kepaptsoglou2009transit,duran2022survey,owais2026transit}. Exact and decomposition-based formulations provide optimization structure for analytical variants of the problem~\cite{bertsimas2021data,vermeir2021exact}. Their tractability, however, depends on simplified demand, assignment, and operating models, which makes simulator-defined objectives with congestion and vehicle capacity difficult to optimize directly. Classical metaheuristics, including genetic algorithms~\cite{mumford2013new}, simulated annealing~\cite{zhao2006simulated}, and bee colony optimization~\cite{nikolic2013transit,duran2022survey}, handle larger design spaces and avoid some of these abstractions, but depend on tailored move operators, penalty terms, and instance-specific tuning. More recently, learning-based approaches reduce manual design by constructing routes with reinforcement learning~\cite{yoo2023reinforcement,darwish2020optimising,li2023transit} or by learning heuristics to guide evolutionary search~\cite{holliday2023augmenting,holliday2024neural,holliday2025learning}. These approaches move toward reusable policies, but most embed learning inside a larger metaheuristic loop rather than producing a standalone construction policy. For methods that construct routes directly, credit assignment remains difficult because the effect of a local route extension on network-wide passenger flow is only observed once routes have been assembled.

Most studies test on small or synthetic benchmark networks~\cite{mandl1980evaluation,heyken2019adaptive} that do not reflect the scale or coupling of realistic urban networks, or rely on analytical objectives and assignment approximations. A smaller line of work uses simulation-based evaluation for bus-oriented transit network redesign and optimization~\cite{manser2020designing,nnene2023simulation}. Both matter because congestion, vehicle capacity, transfer delays, and passenger reassignment are the mechanisms that make TRNDP rewards stochastic, delayed, and nonlocal. Such rewards favor methods with explicit lookahead, where search statistics both guide action selection and serve as training targets for a learned policy~\cite{silver2018general,schrittwieser2020mastering}. Consistent with this motivation, MCTS has been applied in adjacent settings~\cite{kemmerling2023beyond}, including Pareto-optimal transit route planning~\cite{weng2020pareto}, spatial network augmentation~\cite{darvariu2023planning}, surrogate-assisted combinatorial optimization~\cite{amiri2024surrogate}, and metro network expansion combining deep RL with MCTS. The closest neural-search method is MetroZero~\cite{alkilane2025metrozero}, but it selects metro expansion stations under budget constraints, whereas \systemname{} constructs complete bus route sets on existing shared roads and evaluates them through passenger and traffic simulation. However, to the best of our knowledge, no prior work on bus TRNDP brings these three strands together. Learned policy and value priors guide search, MCTS-based lookahead refines action selection, and simulation-based evaluation supplies the terminal reward signal. \systemname{} integrates all three by training policy and value networks against targets produced by MCTS, with rewards drawn from mesoscopic traffic simulation on a city-scale road network~\cite{seo2025uxsim}.

%% file: sections/method.tex
\section{\systemname{}}
\label{sec:method}

\subsection{Problem Formulation}
We model the road network as an undirected graph $G=(V,E)$, where $V=\{v_1,\ldots,v_n\}$ denotes nodes such as intersections and $E$ denotes bidirectional road segments, each with length $\ell_e>0$ and free-flow speed $c_e>0$. Travel demand is exogenous and is represented by an origin--destination matrix $D\in\mathbb{R}_{\ge 0}^{n\times n}$, where $D_{ij}$ gives the number of trips per hour from origin node $i$ to destination node $j$. The modal-split parameter $\alpha\in[0,1]$ assigns a fraction of total OD demand to bus transit, so that $D^{\mathrm{tr}}_{ij}=\alpha\,D_{ij}$. 

A transit route $r_k=(v_{k,1},\ldots,v_{k,L_k})$ is an ordered simple path on $G$: consecutive nodes must share an edge, i.e.\ $\{v_{k,q},v_{k,q+1}\}\in E$, and no node repeats within a route. Routes are operated bidirectionally, so buses traverse both $v_{k,1}\to\cdots\to v_{k,L_k}$ and its reverse. For each completed route set, bus passenger requests are evaluated on the induced transit graph. Passenger itineraries follow shortest-distance paths on this graph and may use one route or transfer across multiple routes to complete their journey. A complete transit design specifies an indexed $K$-tuple of routes, $\Pi=(r_1,\ldots,r_K)$, together with stop spacing and service frequencies:
\begin{equation}
\label{eq:design}
(\Pi,\mathbf{g},\boldsymbol{\mathcal{F}}) \;=\; \bigl((r_1,\ldots,r_K),\ (g_1,\ldots,g_K),\ (\mathcal{F}_1,\ldots,\mathcal{F}_K)\bigr),
\end{equation}
where $g_k\ge 1$ defines the stop spacing for route $k$, and $\mathcal{F}_k\in\mathbb{N}_{>0}$ is the service frequency in buses per hour. The tuple notation provides the index-wise correspondence between route $r_k$, spacing $g_k$, and frequency $\mathcal{F}_k$; the simulator reward is invariant to permutations of route order. Given $G$, $D$, and design constraints such as the route count $K$ and route length bounds, the Transit Route Network Design Problem seeks to maximize expected system-level performance:
\begin{equation}
\label{eq:trndp}
(\Pi^\star,\mathbf{g}^\star,\boldsymbol{\mathcal{F}}^\star) \;\in\; \argmax_{(\Pi,\mathbf{g},\boldsymbol{\mathcal{F}})\in\mathcal{X}} \;\mathbb{E}_{\xi}\!\left[\mathcal{R}\bigl(\Pi,\mathbf{g},\boldsymbol{\mathcal{F}};\xi\bigr)\right],
\end{equation}
where $\mathcal{X}$ is the feasible design set, $\xi$ denotes simulator stochasticity, and $\mathcal{R}(\cdot;\xi)$ is a traffic-simulation performance measure. We fix the stop spacing to $g_k=1$ for all routes and assign frequencies using a max-load rule (Appendix~\ref{app:frequency}); thus, the learned decision variables reduce to the route tuple~$\Pi$.
% where $\mathcal{X}$ is the feasible design set, $\xi$ denotes simulator stochasticity, and $\mathcal{R}(\cdot;\xi)$ is a performance measure evaluated through traffic simulation. We fix the stop spacing to $g_k=1$ for all routes and assign frequencies using a max-load rule (Appendix~\ref{app:frequency}); thus, the learned decision variables reduce to the route tuple~$\Pi$.

We cast the route construction as a finite-horizon Markov decision process $(\mathcal{S},\mathcal{A},P,R_T,s_0)$, where $\mathcal{S}$ is the state space, $\mathcal{A}=\{1,\ldots,n\}$ is the action space with infeasible actions masked at each step, $P$ is the transition function, and $s_0$ is the initial state. Since stop spacing and frequencies are deterministic once $\Pi$ is chosen, the terminal reward is given by 
\begin{equation}
\label{eq:reduced_reward}
R_T(\Pi;\xi) \;=\; \mathcal{R}\bigl(\Pi,\,\mathbf{1},\,\boldsymbol{\mathcal{F}}(\Pi);\,\xi\bigr),
\end{equation}
Here $\boldsymbol{\mathcal{F}}(\Pi)=F^{\mathrm{ml}}(\Pi)$ is the max-load frequency projection defined in Appendix~\ref{app:frequency}. For a completed route set, this projection is the componentwise minimal positive frequency vector satisfying fixed-load capacity constraints. Appendix~\ref{app:frequency} proves this projection property and defines the corresponding gap between the projected route-only objective and joint route-frequency optimization.
Let $\mu_\theta$ be the executed construction policy. For End-to-End RL, $\mu_\theta(a\mid s)=\pi_\theta(a\mid s)$; for \systemname{}, $\mu_\theta$ is the MCTS policy induced by the policy-value network $f_\theta$. The learning objective is
\begin{equation}
\label{eq:objective}
J(\theta)\;=\;\mathbb{E}_{\Pi\sim\mu_\theta,\,\xi}\bigl[R_T(\Pi;\xi)\bigr],
\end{equation}
where $\Pi$ is the complete route set constructed under $\mu_\theta$. In each episode, the agent sequentially constructs $K$ routes, each satisfying $|r_k|\le L_{\max}$, by extending one route at a time. In our setting, $K=16$ and $L_{\max}=14$. Routes start at a transit-center hub and grow one node at a time from the frontier, the most recently appended node, with feasible actions restricted to one-hop neighbors not yet visited in the current route. The policy has no separate stop action; a route is finalized automatically when it reaches $L_{\max}$ or when no valid extension exists. Thus, the learned decision horizon is bounded by $T_{\max}=K(L_{\max}-1)$ and may be shorter when routes terminate early.

\textbf{State.}\; At each step $t$, the state encodes the road graph and the transit network constructed so far, including completed routes, the current partial route, and the frontier node. The policy network receives this as $s_t=(X_t,\mathcal{I},Z)$, where $\mathcal{I}$ is a directed edge list, $Z$ stores edge features such as length and free-flow speed, and $X_t\in\mathbb{R}^{n\times d_x}$ contains per-node features capturing spatial attributes, OD demand aggregates, route membership, and frontier status. The full specification is in Appendix~\ref{app:state}.
% \caption{\systemname{} policy-value network. Graph state consists of node features, edge connectivity, and edge attributes. Node features are projected to embeddings, then passed through a GATv2 backbone. Successive attention blocks produce multi-hop representations, which Jumping Knowledge aggregation concatenates and projects to node embeddings. A node-wise actor MLP produces logits, from which infeasible actions are masked and remaining logits are normalized. A graph-level critic pools node embeddings using global mean and max pooling and predicts the scalar value $V$. \systemname{} uses the resulting policy $P$ as MCTS action priors and $V$ as the leaf-value estimate.}
\begin{figure*}[t!]
    \centering
    \includegraphics[width=0.98\linewidth]{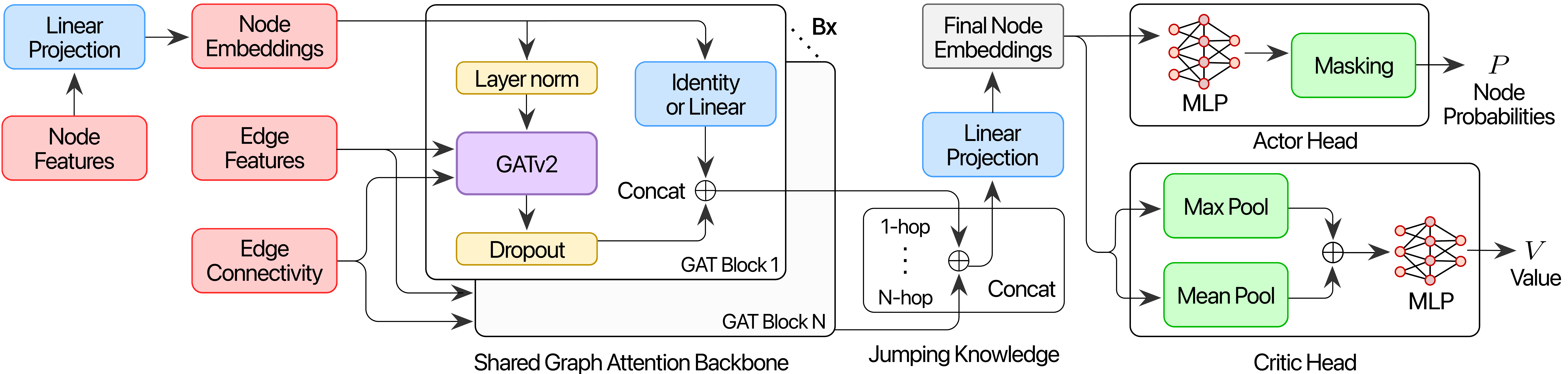}
    \vspace{-4pt}
    \caption{\systemname{} policy-value network. Node features are projected to embeddings, then passed through a GATv2 backbone using edge connectivity and edge attributes. Successive attention blocks produce multi-hop representations, which Jumping Knowledge aggregation concatenates and projects to node embeddings. A node-wise actor MLP produces logits, from which infeasible actions are masked. A graph-level critic uses global mean and max pooling and predicts $V_\theta(s)$. \systemname{} uses $P_\theta(\cdot\mid s)$ as MCTS action priors and $V_\theta(s)$ as the leaf-value estimate.}
    \vspace{-4pt}
    \label{fig:policy_network}
\end{figure*}
\textbf{Action.}\;
At each step the agent selects $a_t\in\mathcal{A}=\{1,\ldots,n\}$, with $a_t=i$ appending $v_i$ to the current route. Let $v_{f_t}$ be the frontier node of partial route $r_{k,t}$. The admissible candidate set is
\begin{equation}
\label{eq:candidate_set}
\mathcal{C}_t
\;=\;
\bigl\{
i\in\{1,\ldots,n\}:
\{v_{f_t},v_i\}\in E
\;\text{and}\;
v_i\notin r_{k,t}
\bigr\},
\end{equation}
so every valid action extends the route by one hop while preserving the simple-path constraint, and $\mathcal{C}_t=\varnothing$ triggers early termination. Because $\mathcal{C}_t$ is restricted to unvisited one-hop neighbors of the frontier and road graphs are sparse, typically $|\mathcal{C}_t|\ll n$. We therefore apply invalid-action masking~\cite{huang2020closer}: a binary mask $m_t\in\{0,1\}^n$ with $m_t[i]=\mathbf{1}\{i\in\mathcal{C}_t\}$ restricts the policy to feasible extensions. Given logits $\ell_\theta(s_t)\in\mathbb{R}^n$, the masked policy, for $\mathcal{C}_t\neq\varnothing$, is
\begin{equation}
\label{eq:masked_policy}
\pi_{\theta}(a_t{=}i\mid s_t)
\;=\;
\begin{cases}
\dfrac{\exp(\ell_i)}
{\sum_{j:\,m_t[j]=1}\exp(\ell_j)}
& \text{if } m_t[i]=1,\\[8pt]
0
& \text{if } m_t[i]=0.
\end{cases}
\end{equation}
Node feasibility is also exposed in $X_t$ via a valid-next flag, allowing the policy to observe which actions are admissible while the mask enforces zero probability on invalid actions.

% Since $D^{\mathrm{tr}}_{ij}=\alpha D_{ij}$, $\Psi$ is upper-bounded by $\alpha$.
\textbf{Reward.}\;
For a complete route design $\Pi=(r_1,\ldots,r_K)$, we assign frequencies $\boldsymbol{\mathcal{F}}(\Pi)$ via the max-load rule and evaluate the design through simulation. The terminal reward for training is
\begin{equation}
\label{eq:reward}
R_T(\Pi;\xi) = b_0\,\Psi + b_1\,\rho - b_2\,\hat{t}_{\mathrm{wait}} - b_3\,\hat{t}_{\mathrm{move}} - b_4\,\omega - b_5\,\tfrac{N_{\mathrm{bus}}}{K} + b_6\,u,
\end{equation}
with $b_0{=}60,\;b_1{=}45,\;b_2{=}20,\;b_3{=}10,\;b_4{=}10,\;b_5{=}2,\;b_6{=}12$, chosen to balance passenger and operator objectives; planners with different goals can set different weights. The coverage potential $\Psi$ is the fraction of total OD demand assigned to transit and reachable under $\Pi$. Let $G_\Pi=(V_\Pi,E_\Pi)$ be the graph induced by $\Pi$, with $V_\Pi$ as served nodes and $E_\Pi$ as edges. Then
\begin{equation}
\label{eq:coverage}
\Psi
\;=\;
\frac{\sum_{(i,j)\in\mathcal{P}_{\mathrm{reach}}} D^{\mathrm{tr}}_{ij}}
{\sum_{i,j\in V} D_{ij}},
\qquad
\mathcal{P}_{\mathrm{reach}}
=
\{(i,j): i,j\in V_\Pi \text{ and } j \text{ is reachable from } i \text{ in } G_\Pi\}.
\end{equation}
The service term is $\rho=N_{\mathrm{boarded}}/N_{\mathrm{OD}}$, where $N_{\mathrm{boarded}}$ counts passengers who complete trips or are onboard at the simulation horizon and $N_{\mathrm{OD}}$ is citywide OD demand over that horizon. This fixed-denominator reward term differs from the reported service rate $\sigma=N_{\mathrm{boarded}}/N_{\mathrm{want}}$ in Appendix~\ref{app:metrics}. Let $\bar{t}_{\mathrm{wait}}$ and $\bar{t}_{\mathrm{move}}$ denote raw average waiting and in-vehicle times in minutes over served passengers. The reward uses capped normalized time penalties,
\[
\hat{t}_{\mathrm{wait}}=\min\{\bar{t}_{\mathrm{wait}}/30,1\},\qquad
\hat{t}_{\mathrm{move}}=\min\{\bar{t}_{\mathrm{move}}/40,1\}.
\]
The remaining terms are route overlap ratio $\omega$, fleet size $N_{\mathrm{bus}}$, and bus utilization $u$. The scalar terminal reward is separately normalized online when forming the value target~\cite{huang202237, rl_bag_of_tricks}.
% \begin{figure*}[t!]
%     \centering
%     \includegraphics[width=0.98\linewidth]{figures/final_learning_overview_alpha_0_3.pdf}
%     \vspace{-6pt}
%     \caption{Learning dynamics and search cost under mixed demand ($\alpha=0.3$). \textbf{LEFT}: Curves show averages over two training seeds per reward mode. Reward shaping is critical for End-to-End RL. Early Stopping (ES) penalizes routes that terminate before $L_{\max}$, while Delta-coverage ($\Delta$) rewards only newly covered demand. The $\Delta$ variants are strongest; $\Delta$ + No ES is best. \textbf{MIDDLE}: MCTS search improves sample efficiency under the same environment-step budget. At search depth $N_{\mathrm{iter}}=500$, \systemname{} surpasses End-to-End RL around $3\times10^5$ steps and finishes with a final smoothed reward $1.31$ versus $-4.02$. \textbf{RIGHT}: Learned policy and value estimates make MCTS search practical across search depths $N_{\mathrm{iter}}\in[100,500]$, when benchmarked on the Bloomington network with a single CPU worker. \systemname{} stays within seconds per decision, whereas Pure MCTS requires hundreds to thousands of seconds per decision.}
%     \vspace{-4pt}
%     \label{fig:learning_overview_0_3}
% \end{figure*}
\subsection{Search-Guided Reinforcement Learning}
\label{sec:alphatransit}
Sequential transit design provides a sparse training signal because node extensions are evaluated only after route designs are completed, frequencies are assigned, and passenger flows are simulated. Local construction policies, including heuristics and end-to-end RL, select each node without looking ahead to downstream network effects at decision time; these effects are observed only after design evaluation. \systemname{} runs Monte Carlo Tree Search (MCTS) at each construction state, guided by the policy-value network $f_\theta$. The visit counts define an MCTS policy used both to sample the next node and to train the actor. This makes \systemname{} a search-guided reinforcement learning method.

 % Top panels show training progress, and bottom panels show top-$5$ reward versus compute budget.
\begin{figure*}[t!]
    \centering
    \includegraphics[width=0.98\linewidth]{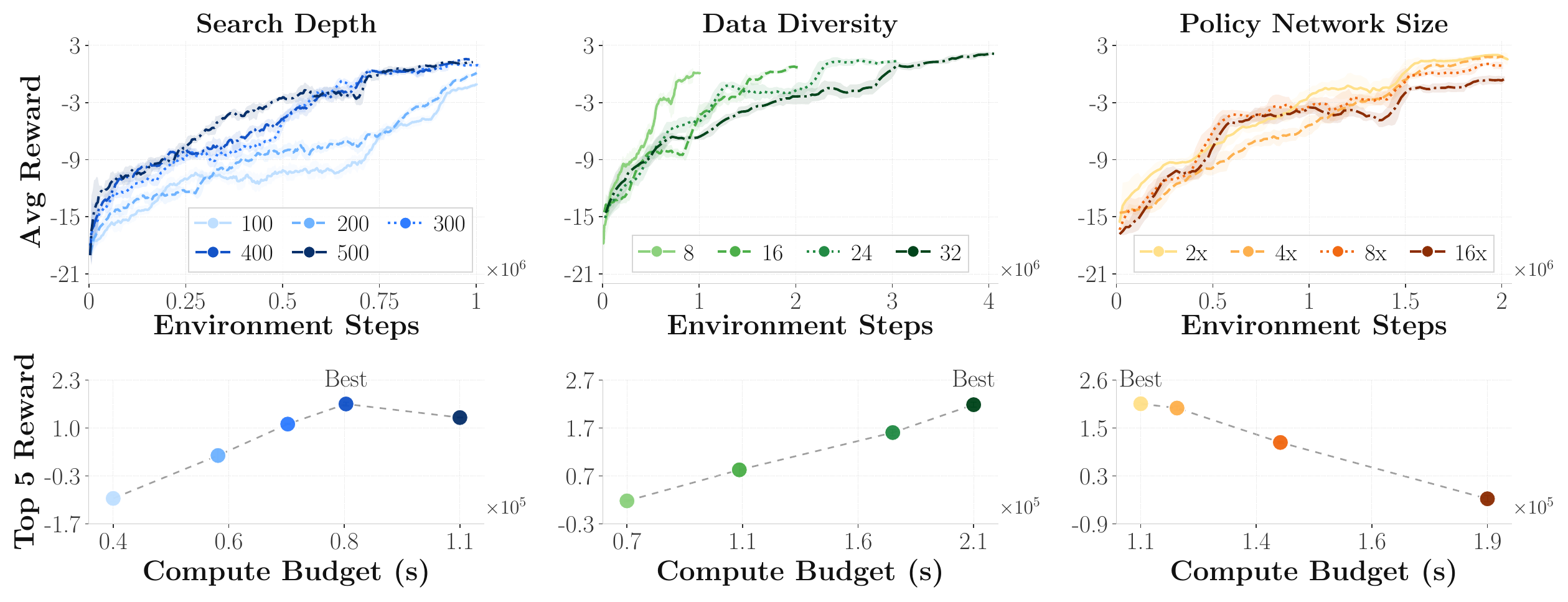}
    \vspace{-6pt}
    \caption{\systemname{} scaling behavior under mixed demand ($\alpha=0.3$). Search depth, data diversity, and policy size show different quality--compute trade-offs. \textbf{LEFT}: Performance peaks at $N_{\mathrm{iter}}=400$ with top-$5$ reward $1.67$, rather than increasing monotonically, while $N_{\mathrm{iter}}=500$ lowers reward to $1.29$ with higher runtime. \textbf{MIDDLE}: Increasing episodes per iteration gives the highest reward, $2.20$ at $32$ episodes, but costs the most compute. \textbf{RIGHT}: Policy-size labels ($2\times$-$16\times$) denote the GAT-block repetition count; the $2\times$ policy reaches $2.05$, while the $16\times$ policy falls below $0$.}
    \vspace{-4pt}
    \label{fig:scaling_behavior_0_3}
\end{figure*}
% \caption{\systemname{} scaling behavior under mixed demand ($\alpha=0.3$). Search depth, data diversity, and policy size show different quality--compute profiles. Search peaks at $N_{\mathrm{iter}}=400$, reaching top-$5$ reward $1.67$, while $N_{\mathrm{iter}}=500$ lowers reward to $1.29$ with higher runtime. Increasing episodes per iteration gives the highest reward, $2.20$ at $32$ episodes per iteration, but incurs the largest compute cost. Policy size ($2\times$-$16\times$) indicate GAT-block repetition factors, $2\times$ policy reaches $2.05$, while the $16\times$ policy falls to $-0.29$.}

Fig.~\ref{fig:policy_network} shows $f_\theta$ mapping graph state $s=(X,\mathcal{I},Z)$ to $f_\theta(s)=(P_\theta(\cdot\mid s),V_\theta(s))$, where $P_\theta$ gives node-action probabilities and $V_\theta$ estimates the value of completing a design from $s$. The network projects node features and processes them with a shared GATv2~\cite{brody2021attentive} backbone using $\mathcal{I}$ and $Z$. The actor and critic share this backbone. For fixed hidden widths, attention heads, and block count $B$, the parameter count is independent of $n$, while a forward pass costs $O(B(n+|\mathcal{I}|))$, linear in $n$ for sparse road graphs. Jumping Knowledge aggregation~\cite{xu2018jk} gives both heads access to multiple receptive-field scales. Let $\mathbf{h}_v^{(b)}$ denote the representation of node $v$ after block $b$. The block outputs are concatenated and projected to a node embedding: 
\begin{equation}
    \mathbf{z}_v = W_{\mathrm{JK}}\left[\mathbf{h}_v^{(1)}\,\Vert\, \cdots \,\Vert\, \mathbf{h}_v^{(B)}\right]+\mathbf{b}_{\mathrm{JK}},
    \label{eq:jk}
\end{equation}
where $W_{\mathrm{JK}}\in\mathbb{R}^{d\times\sum_{b=1}^{B}d_b}$ maps the concatenation to dimension $d$. The actor applies the same MLP to each $\mathbf{z}_v$, producing one logit per node, then masks infeasible logits via Eq.~\ref{eq:masked_policy}. The critic concatenates global mean and max pooled node embeddings and maps the pooled representation to $V_\theta(s)$. Appendix~\ref{app:policy} provides architecture details.

% \textbf{Monte Carlo Tree Search.}\; 
At each construction state $s_t$, \systemname{} runs an MCTS rooted at $s_t$ over admissible route-extension actions. For any tree state $s$, let $\mathcal{C}(s)$ be the candidate set from Eq.~\ref{eq:candidate_set}; at the root, $\mathcal{C}(s_t)=\mathcal{C}_t$. Each edge $(s,a)$, with $a\in\mathcal{C}(s)$, stores $N(s,a)$, $W(s,a)$, $Q(s,a)$, and prior $P_\theta(a\mid s)$, where $Q(s,a)=W(s,a)/N(s,a)$ if $N(s,a)>0$ and $Q(s,a)=0$ otherwise. Leaf states are evaluated by $V_\theta$, avoiding simulator calls inside the tree. Each simulation then proceeds through:
\vspace{-2pt}
\begin{itemize}[leftmargin=12pt]
    \item Selection:\;
    Starting at the root, each simulation selects actions using the PUCT rule~\cite{rosin2011multi}, which combines the current value estimate with an exploration bonus:
    \begin{equation}
    a^\star = \mathop{\argmax}_{a\in\mathcal{C}(s)} \left[ Q(s,a) + c\,P_\theta(a\mid s) \frac{\sqrt{1+\sum_{b\in\mathcal{C}(s)}N(s,b)}}{1+N(s,a)} \right].
    \label{eq:puct}
    \end{equation}
    Here, $c$ balances exploitation of high-value actions with exploration of less-visited ones.

    \item Expansion and evaluation:\;
    When selection reaches an unexpanded leaf $s_{\mathrm{leaf}}$, a single forward pass of $f_\theta$ produces priors $P_\theta(a\mid s_{\mathrm{leaf}})$ for each $a\in\mathcal{C}(s_{\mathrm{leaf}})$ and a scalar value $V_\theta(s_{\mathrm{leaf}})$ used as the backed-up leaf value. No simulator rollout is performed inside the tree; the simulator is invoked only after the executed trajectory completes a full route set.
    
    \item Backpropagation:\;
    After evaluation, statistics are updated along the path from leaf to root: $N(s,a)\leftarrow N(s,a)+1$, $W(s,a)\leftarrow W(s,a)+V_\theta(s_{\mathrm{leaf}})$, and $Q(s,a)\leftarrow W(s,a)/N(s,a)$.
\end{itemize}
After $N_{\mathrm{iter}}$ simulations, the action distribution at the root is derived from visit counts:
\begin{equation}
\pi_t(a\mid s_t) = \frac{N(s_t,a)^{1/\tau}}{\sum_{b\in\mathcal{C}(s_t)} N(s_t,b)^{1/\tau}}, \qquad a\in\mathcal{C}(s_t),
\label{eq:mcts_policy}
\end{equation}
where $\tau$ is the temperature parameter. During training, the next action is sampled from $\pi_t$, and the pair $(s_t,\pi_t)$ is stored for the episode. Once the full route set $\Pi$ is constructed, the simulator returns the terminal reward $z=R_T(\Pi;\xi)$, which is normalized online to form the value target $\tilde{z}$ and paired with each stored state. The network is trained from samples $(s,\pi,\tilde{z})$ by minimizing
\begin{equation}
\mathcal{L}(\theta) = \mathbb{E}_{(s,\pi,\tilde{z})} \left[ - \sum_{a\in\mathcal{C}(s)} \pi(a\mid s)\log P_\theta(a\mid s) + \bigl(V_\theta(s)-\tilde{z}\bigr)^2 \right].
\label{eq:alphatransit_loss}
\end{equation}
The first term distills MCTS visit counts into the actor, and the second term trains the critic to predict the normalized terminal outcome. The full training algorithm is provided in Appendix~\ref{app:algorithm}.

%% file: sections/experiment.tex
\section{Experiment}
\label{sec:experiment}
\textbf{City Networks and Demand Data.}\;
We introduce a new real-world transit design benchmark dataset that includes: (i) a topologically correct road graph with $143$ nodes and $243$ bidirectional edges, extracted from the Bloomington street network (${\sim}152.3$\,km$^2$); (ii) a block-level OD demand matrix derived from U.S. Census data; and (iii) the $16$ existing real-world transit routes~\cite{bloomingtontransit_gtfs}. Unlike synthetic benchmarks, our dataset provides realistic road topology and spatially heterogeneous travel demand. We train and evaluate \systemname{} on this network. Processing details are in Appendix~\ref{app:bloomington}. To evaluate cross-city generalization, we test policies trained on the Bloomington network on a larger Laval network from Holliday et al.~\cite{holliday2025learning} (${\sim}256$\,km$^2$ with $632$ nodes and $1{,}971$ edges), using the same design parameters, $K{=}16$ and $L_{\max}{=}14$. Additional details are in Appendix~\ref{app:laval}.

\begin{figure*}[t!]
    \centering
    \includegraphics[width=0.98\linewidth]{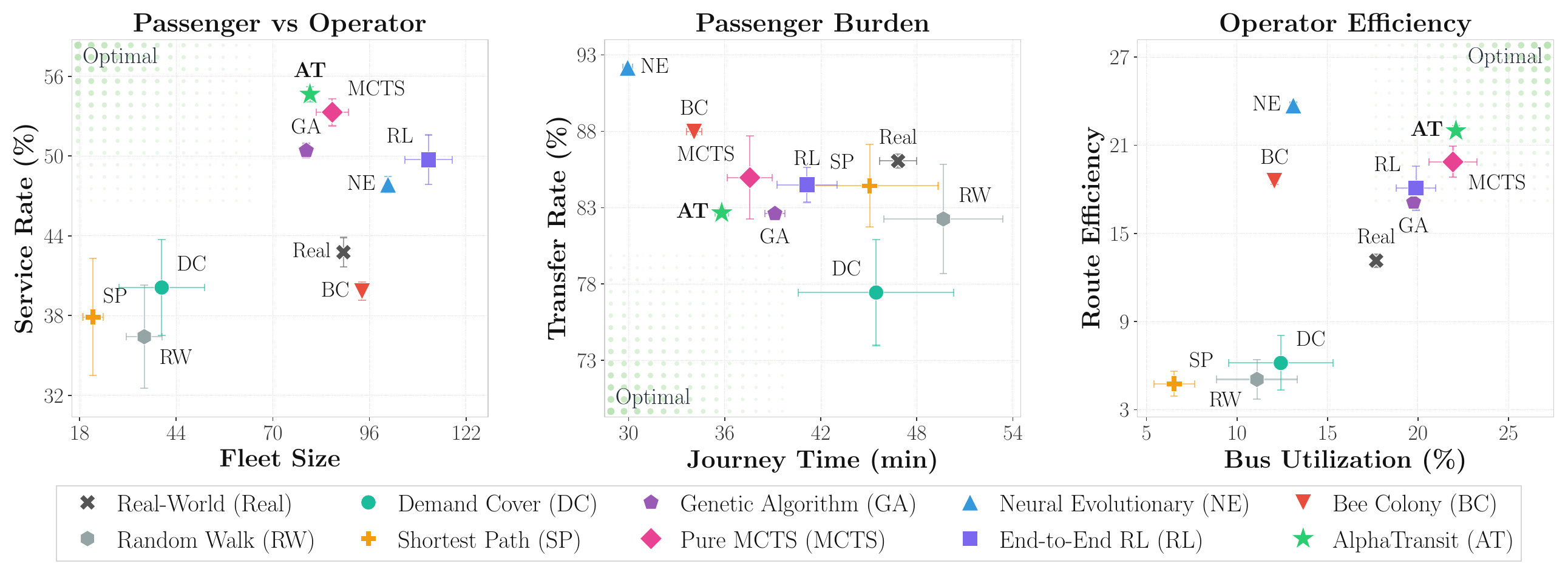}
    \vspace{-8pt}
    \caption{Mixed-demand results on the Bloomington benchmark ($\alpha=0.3$). Points show means; error bars denote $\pm 1$ standard deviation. In each panel, green overlay and Optimal label mark improvement direction: upper-left for service rate versus fleet size, lower-left for journey time versus transfer rate, and upper-right for bus utilization versus route efficiency. \textbf{LEFT}: \systemname{} achieves highest service rate, $54.64\%$, with fleet size $80$. \textbf{MIDDLE}: \systemname{} obtains $35.81$ minutes of journey time and an $82.66\%$ transfer rate. \textbf{RIGHT}: \systemname{} achieves highest bus utilization, $22.10\%$, and the second-highest route efficiency, $21.99$. Together, the panels show that \systemname{} pairs the highest service rate and bus utilization with a competitive passenger-burden trade-off.}
    \vspace{-4pt}
    \label{fig:comparison_0_3}
\end{figure*}

\textbf{Setup.}\;
% \subsection{Setup}
We run simulations in UXsim~\cite{seo2025uxsim}, a mesoscopic traffic simulator based on Newell's car-following model~\cite{newell2002simplified}. Its mesoscopic resolution captures congestion propagation while remaining tractable for iterative optimization: by aggregating vehicles into platoons of size $\Delta n=5$, UXsim runs 30--60$\times$ faster than microscopic simulators such as SUMO~\cite{lopez2018microscopic}. Dynamics advance at $\Delta t=1$\,s for $T_{\text{sim}}=10{,}000$ steps ($\approx 2.7$ hours), covering a representative morning peak~\cite{cambridge2005congestion}. We extend UXsim with a training environment that handles bus dispatch, boarding, alighting, and modal split. Buses operate with $40$-passenger capacity and $60$\,s dwell time per stop, while $\alpha$ allocates OD demand between bus and car modes. We train \systemname{} on Bloomington for approximately $1$M environment steps using parallel episode workers and $N_{\mathrm{iter}}=500$ MCTS simulations per decision in the final comparisons. Unless stated otherwise, all \systemname{} and Pure MCTS final comparisons use this search budget. Each episode constructs $16$ routes, grown node by node until reaching $L_{\max}=14$ stops or no valid one-hop extension remains. All routes begin at the transit center hub, matching real-world Bloomington Transit. We consider two modal splits: $\alpha=1.0$, which assigns all served demand to bus transit, and $\alpha=0.3$, which reflects a typical urban public-transport share~\cite{buehler2019verkehrsverbund}. We train one \systemname{} policy per scenario. All completed route sets are evaluated through the same UXsim reporting pipeline, with frequencies assigned by the same max-load rule and metrics derived from Eq.~\ref{eq:reward}. Training, hyperparameters, and baseline details are in Appendix~\ref{app:training}. We compare against the following baselines:
\begin{figure*}[t!]
    \centering
    \includegraphics[width=0.98\linewidth]{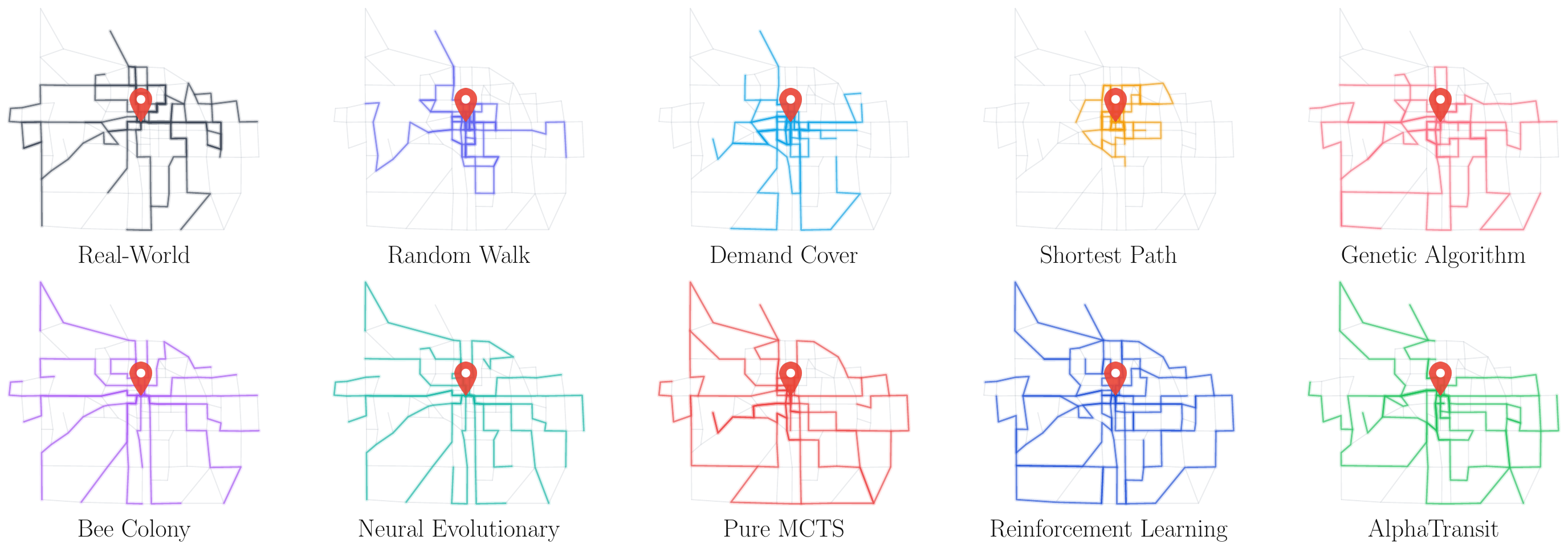}
    \vspace{-6pt}
    \caption{Selected route designs under mixed demand ($\alpha=0.3$) on the Bloomington network, with the red pin marking the transit center. Each panel overlays a method's routes on the same street basemap. \systemname{} covers $117$ nodes ($81.8\%$), has $24.4\%$ shared-edge overlap, and spans $120.8$ km. Real-World covers $114$ nodes ($79.7\%$), has $34.8\%$ overlap, and spans $115.6$ km, while End-to-End RL covers $138$ nodes ($96.5\%$), has $19.0\%$ overlap, and spans $142.3$ km. \systemname{} is therefore more targeted than End-to-End RL.}
    \label{fig:routes_0_3}
    \vspace{-6pt}
\end{figure*}
\vspace{-2pt}
\begin{itemize}[leftmargin=12pt]
    \setlength{\itemsep}{0pt}
    \setlength{\parskip}{0pt}
    \item Real-World: The $16$ bus routes operated by Bloomington Transit~\cite{bloomingtontransit_gtfs}. The agency's design reflects broader objectives such as equity, coverage, and budget that extend beyond our simulator-defined criteria; we include it as a real-world reference.
    
    \item Random Walk: Samples next node uniformly from the admissible candidate set $\mathcal{C}_t$.
    
    \item Demand Coverage: Samples next node $i \in \mathcal{C}_t$ proportional to its demand interaction with the partial route, $s(i) = \sum_{j \in r_k} (D_{ij} + D_{ji})$.
    
    \item Shortest Path: Samples next node $i \in \mathcal{C}_t$ inversely proportional to edge length, $s(i) = 1/\ell_{u,i}$, where $u$ is the frontier and $\ell_{u,i}$ is the length of edge $\{u, i\}$.
    
    \item Genetic Algorithm: A widely used metaheuristic for transit network design~\cite{fan2006optimal,nayeem2014transit}. Each individual represents a complete route tuple $\Pi=(r_1,\ldots,r_K)$ that evolves over generations through tournament selection, route-exchange crossover, and path-regeneration mutation.
    
    \item Bee Colony Optimization: A swarm-intelligence metaheuristic~\cite{nikolic2013transit} that evolves a population of route sets through two mutations: (i) demand-weighted shortest-path route replacement and (ii) node addition/deletion at route endpoints. We adapt the implementation from~\cite{holliday2025learning}, preserve the analytical objective during route generation, and evaluate the completed routes through UXsim.
    
    \item Neural Evolutionary Algorithm: Augments Bee Colony Optimization with mutations proposed by a graph neural policy trained via RL~\cite{holliday2024neural,holliday2025learning}. We train the policy under our setup, preserving the original analytical objective during search, and evaluate the completed routes through UXsim.

    \item Pure Monte Carlo Tree Search~\cite{coulom2006efficient,browne2012survey}: Combines uniform action priors $P(s,a)=1/|\mathcal{C}_t|$ and full UXsim rollouts. The PUCT formula (Eq.~\ref{eq:puct}) and search budget match \systemname{}.
    
    \item End-to-End RL: A direct policy-learning baseline trained with PPO~\cite{schulman2017proximal} and GAE~\cite{schulman2015high}. It selects actions from the masked policy $\pi_\theta(a \mid s_t)$ without decision-time search. It uses the same state, action mask, terminal reward, and policy architecture as \systemname{}, plus lightweight shaping rewards with $\gamma=0.999$ to make long horizon policy gradient training feasible.
\end{itemize}

% Cross-city evaluation on the Laval network. Under full transit demand ($\alpha=1.0$), \systemname{} reaches $90.72\%$ service rate, compared with $55.03\%$ for End-to-End RL, and also obtains the lowest wait time and highest route efficiency among the transfer-compatible methods. Under mixed demand ($\alpha=0.3$), \systemname{} remains close to Shortest Path in service rate while giving the lowest wait time. Demand Cover and Shortest Path are non-learned reference heuristics; target-instance optimizers are omitted because they do not evaluate transfer of a fixed learned policy. Values report mean $\pm$ standard deviation over $10$ seeds.
% We evaluate transit network designs along three axes grounded in standard transit evaluation criteria~\cite{vuchic2007urban, vuchic2017urban}. The first axis weighs passenger benefit against operator cost: service rate (\% of potential demand boarded) versus fleet size. The second axis groups two passenger-experience metrics: total journey time (minutes), measured as elapsed passenger time including waiting and in-vehicle movement, and transfer rate (\% of trips requiring a transfer). The third axis groups two operator-side efficiency metrics: route efficiency (passengers served per km of route) and bus utilization (\% of bus capacity occupied). Full definitions are provided in Appendix~\ref{app:metrics}.

We evaluate network designs across three axes from transit evaluation criteria~\cite{vuchic2007urban, vuchic2017urban}. The first weighs passenger benefit against operator cost: service rate (\% of potential demand boarded) vs. fleet size. Here, potential demand refers to the reachable fixed-transit-demand denominator $N_{\text{want}}$ defined in Appendix~\ref{app:metrics}. The second groups two passenger-experience metrics: total journey time (minutes), elapsed passenger time including waiting and in-vehicle movement, and transfer rate (\% of trips requiring a transfer). The third groups two operator-side efficiency metrics: route efficiency (passengers served per km of route) and bus utilization (\% of bus occupied). Full definitions are provided in Appendix~\ref{app:metrics}.

\begin{table*}[t!]
  \centering
  % \scriptsize
  \renewcommand{\arraystretch}{1.15}
  \resizebox{0.99\textwidth}{!}{%
    \begin{tabular}{cl ccccccc}
        & & \multicolumn{4}{c}{Passenger Metrics} & \multicolumn{3}{c}{Operator Metrics} \\
        \cmidrule(lr){3-6} \cmidrule(lr){7-9}
        & Method
        & \makecell{Service \\ Rate (\%) $\uparrow$}
        & \makecell{Wait \\ Time (min) $\downarrow$}
        & \makecell{Transfer \\ Rate (\%) $\downarrow$}
        & \makecell{Journey \\ Time (min) $\downarrow$}
        & \makecell{Route \\ Efficiency $\uparrow$}
        & \makecell{Fleet \\ Size $\downarrow$}
        & \makecell{Bus \\ Util. (\%) $\uparrow$} \\
        \midrule
        \multirow{4}{*}{\rotatebox{90}{$\alpha = 0.3$}}
        & Demand Cover & $73.80 {\scriptstyle \pm 5.40}$ & $15.93 {\scriptstyle \pm 2.89}$ & $49.73 {\scriptstyle \pm 4.64}$ & $41.03 {\scriptstyle \pm 3.71}$ & $103.30 {\scriptstyle \pm 18.39}$ & $96.10 {\scriptstyle \pm 16.23}$ & $38.03 {\scriptstyle \pm 3.93}$ \\
        & Shortest Path & $90.72 {\scriptstyle \pm 7.53}$ & $9.91 {\scriptstyle \pm 3.12}$ & $57.48 {\scriptstyle \pm 4.03}$ & $29.62 {\scriptstyle \pm 2.80}$ & $129.85 {\scriptstyle \pm 37.78}$ & $72.20 {\scriptstyle \pm 21.67}$ & $23.77 {\scriptstyle \pm 3.63}$ \\
        & End-to-End RL & $88.96 {\scriptstyle \pm 2.01}$ & $7.38 {\scriptstyle \pm 0.71}$ & $58.91 {\scriptstyle \pm 1.12}$ & $39.99 {\scriptstyle \pm 1.21}$ & $225.69 {\scriptstyle \pm 9.93}$ & $277.00 {\scriptstyle \pm 0.00}$ & $47.34 {\scriptstyle \pm 3.14}$ \\
        & AlphaTransit & $89.25 {\scriptstyle \pm 0.84}$ & $7.07 {\scriptstyle \pm 0.54}$ & $57.70 {\scriptstyle \pm 0.70}$ & $38.19 {\scriptstyle \pm 1.00}$ & $200.87 {\scriptstyle \pm 4.77}$ & $241.00 {\scriptstyle \pm 0.00}$ & $42.14 {\scriptstyle \pm 1.18}$ \\
        \midrule
        \multirow{4}{*}{\rotatebox{90}{$\alpha = 1.0$}}
        & Demand Cover & $62.00 {\scriptstyle \pm 7.96}$ & $12.41 {\scriptstyle \pm 4.65}$ & $41.28 {\scriptstyle \pm 5.80}$ & $45.37 {\scriptstyle \pm 6.01}$ & $256.48 {\scriptstyle \pm 45.16}$ & $270.90 {\scriptstyle \pm 32.43}$ & $45.87 {\scriptstyle \pm 4.34}$ \\
        & Shortest Path & $79.69 {\scriptstyle \pm 16.75}$ & $7.21 {\scriptstyle \pm 3.17}$ & $49.38 {\scriptstyle \pm 10.86}$ & $33.08 {\scriptstyle \pm 6.32}$ & $324.53 {\scriptstyle \pm 101.56}$ & $190.60 {\scriptstyle \pm 55.88}$ & $30.55 {\scriptstyle \pm 5.37}$ \\
        & End-to-End RL & $55.03 {\scriptstyle \pm 1.79}$ & $15.80 {\scriptstyle \pm 0.74}$ & $46.76 {\scriptstyle \pm 0.97}$ & $50.68 {\scriptstyle \pm 1.25}$ & $313.00 {\scriptstyle \pm 7.73}$ & $633.00 {\scriptstyle \pm 0.00}$ & $54.79 {\scriptstyle \pm 1.07}$ \\
        & AlphaTransit & $90.72 {\scriptstyle \pm 0.71}$ & $6.03 {\scriptstyle \pm 0.58}$ & $63.33 {\scriptstyle \pm 0.48}$ & $35.50 {\scriptstyle \pm 0.61}$ & $396.30 {\scriptstyle \pm 6.27}$ & $330.00 {\scriptstyle \pm 0.00}$ & $44.34 {\scriptstyle \pm 0.32}$ \\
        \bottomrule
    \end{tabular}%
  }
  \vspace{3pt}
  \caption{Cross-city transfer on Laval. Values report mean $\pm$ standard deviation over $10$ seeds; Laval-optimized learning/search methods are omitted to isolate fixed-policy transfer. Under full transit demand ($\alpha=1.0$), \systemname{} reaches $90.72\%$ service rate, versus $55.03\%$ for End-to-End RL, and also obtains the lowest wait time and highest route efficiency among transfer-compatible methods, giving the clearest advantage. Under mixed demand ($\alpha=0.3$), \systemname{} remains close to Shortest Path in service rate while giving the lowest wait time.}
  \label{tab:generalization_laval_compact}
  \vspace{-8pt}
\end{table*}

\textbf{Results.}\;
% \subsection{Results}
Fig.~\ref{fig:learning_overview_0_3} summarizes the learning dynamics and decision-time search cost under mixed demand. End-to-End RL relies on dense shaping to learn over the long route-construction horizon: the Delta-coverage variants outperform terminal-only and Raw + Early Stopping feedback, with Delta + No Early Stopping used in later comparisons. \systemname{} improves sample efficiency under the same environment-step budget because MCTS provides search-improved action targets, while learned policy-value estimates keep decision-time search in seconds rather than the hundreds to thousands of seconds required by Pure MCTS. Fig.~\ref{fig:scaling_behavior_0_3} focuses on \systemname{} scaling: more compute helps only when allocated carefully. Performance improves with search depth up to $N_{\mathrm{iter}}=400$ and drops at $500$; episodes per iteration raise top-5 reward from $0.19$ to $2.20$; larger GAT policies provide no gain.

Fig.~\ref{fig:comparison_0_3} gives the mixed-demand ($\alpha=0.3$) comparison on the Bloomington benchmark, and Appendix~\ref{app:results} reports the complete metric table and complementary analyses. Under mixed demand, \systemname{} achieves the highest service rate, $54.64\%$, with $80$ buses and the highest bus utilization, $22.10\%$. The same service-rate ranking holds under full transit demand ($\alpha=1.0$), where \systemname{} reaches $82.08\%$ service rate and also obtains the best wait time, route efficiency, and bus utilization. Since these metrics capture different passenger and operator objectives, no method dominates every axis: \systemname{} is strongest on service rate and utilization, while other methods can obtain lower transfer rates, shorter journey times, or smaller fleets in some settings.

To test whether the gain comes from combining learning with search, we compare against End-to-End RL, which removes decision-time search, and Pure MCTS, which uses the same search budget without learned prior-value estimates. Relative to End-to-End RL, \systemname{} improves service rate by $9.9\%$ at $\alpha=0.3$ and $11.4\%$ at $\alpha=1.0$; relative to Pure MCTS, it improves service rate by $2.5\%$ and $11.2\%$, respectively. At $\alpha=0.3$, Pure MCTS reaches $53.30\%$ with $86$ buses, whereas \systemname{} reaches $54.64\%$ with $80$ buses, suggesting search works best with learned estimates. Fig.~\ref{fig:routes_0_3} shows that the gain is not simply broader coverage: \systemname{} covers fewer nodes and less route distance than End-to-End RL ($117$ vs. $138$ nodes; $120.8$ km vs. $142.3$ km), yet serves more demand with a smaller fleet. Compared with Real-World routes, it has similar node coverage but reduces shared edge overlap ($24.4\%$ vs. $34.8\%$) and improves service rate from $42.77\%$ to $54.64\%$.

Finally, Table~\ref{tab:generalization_laval_compact} evaluates cross-city transfer on the larger Laval network using policies trained only on Bloomington. Under full transit demand ($\alpha=1.0$), \systemname{} reaches $90.72\%$ service rate, compared with $55.03\%$ for End-to-End RL, and also obtains the lowest wait time and highest route efficiency among the listed transfer-compatible methods. Under mixed demand ($\alpha=0.3$), Shortest Path has the highest service rate, while \systemname{} remains close at $89.25\%$ and gives the lowest wait time. The Laval results are mixed across metrics, but they show that the search-guided policy transfers to a larger network, with the clearest advantage in the higher-demand setting.

%% file: sections/conclusion.tex
\section{Conclusion and Discussion}
\label{sec:conclusion}

Transit route network design is a sequential combinatorial problem in which each route-extension decision is judged only after the full network is assembled and simulated. This delayed and nonlocal feedback makes TRNDP deceptive, since extensions that look useful locally can later create redundant overlap, transfer burden, or capacity bottlenecks. AlphaTransit addresses this challenge by coupling Monte Carlo Tree Search with a graph attention policy-value network. The policy proposes feasible route extensions, the value estimates downstream design quality, and search refines each decision without simulator rollouts inside the tree.

We also introduce a new Bloomington benchmark with realistic road topology, census-derived origin-destination demand, and existing transit routes. The results show that AlphaTransit achieves the highest service rate in both mixed and full transit demand, reaching $54.64\%$ at $\alpha=0.3$ and $82.08\%$ at $\alpha=1.0$. Relative to End-to-End Reinforcement Learning, AlphaTransit improves service rate by $9.9\%$ and $11.4\%$; relative to Pure Monte Carlo Tree Search, it improves service rate by $2.5\%$ and $11.2\%$. The scaling results further show that calibrated search depth and episode diversity matter more than simply enlarging the policy network. Together, these results establish learned lookahead as an effective mechanism for coordinating route construction under delayed, network-level evaluation. 

Several limitations of the current work exist. First, the geographic scope is limited: most experiments and model development rely on the Bloomington benchmark, so the results may not fully capture the diversity of transit networks across cities. The Laval experiment provides an initial cross-city check, but broader validation on additional metropolitan systems is needed. A second limitation is the route-start assumption. Every route begins at the transit-center hub, which matches Bloomington Transit and reduces the construction space, but may not fit multi-hub, crosstown, or grid-like networks. Finally, the simulator and decision model abstract several deployment factors: demand is represented by a static peak-hour OD matrix, and the reward does not explicitly encode equity, accessibility, reliability, budget, or robustness constraints.

Future work should extend training and evaluation with endogenous mode choice and to more cities and larger metropolitan networks, relax the transit-center start constraint, or learn route origins jointly with route extensions. Other promising extensions include time-varying demand, stochastic disruptions, a second-stage frequency policy, and constrained objectives that more clearly expose service-quality trade-offs across neighborhoods.

%% file: appendix_sections/search_space.tex
\section{Size of the search space}
\label{app:searchspace}
The scope of the design problem is to choose $K=16$ routes, each visiting $14$ distinct nodes, on the Bloomington network graph $G=(V, E)$. Computing the exact size of the search space is computationally intensive, so we estimate it by approximating the number of simple paths (no repeated nodes) consisting of $L=13$ edges. The graph $G$ has $|V|=143$ and $|E|=243$, which gives an average degree
\[
d \;=\; \frac{2|E|}{|V|} \;=\; \frac{486}{143} \;\approx\; 3.40.
\]
\noindent\textbf{Random initialization.}\;
For simple paths in sparse graphs, we approximate the number of candidate routes by accounting for the constraint that nodes cannot repeat. If each route may start from any of $|V|$ nodes, the first step has approximately $d$ choices, and each subsequent step has approximately $d-1$ choices (i.e., excluding the node just visited):
\[
P_{\mathrm{random}} \;\approx\; |V| \cdot d \cdot (d-1)^{L-1}, \qquad
S_{\mathrm{random}} \;=\; P_{\mathrm{random}}^{K}.
\]

Numerically,
\[
(d-1)^{12} = (2.40)^{12} \approx 3.65\times 10^{4},~
P_{\mathrm{random}} \approx 1.78\times 10^{7},
S_{\mathrm{random}} \approx 9.5\times 10^{115} \approx 10^{116}.
\]

\noindent\textbf{Transit-center initialization.}\;
In the experiments reported in this paper, all routes start from the transit-center hub. This removes the $|V|$ factor from each route count:
\[
P_{\mathrm{hub}} \;\approx\; d \cdot (d-1)^{L-1}, \qquad
S_{\mathrm{hub}} \;=\; P_{\mathrm{hub}}^{K}.
\]
Numerically,
\[
P_{\mathrm{hub}} \approx 1.24\times 10^{5},\qquad
S_{\mathrm{hub}} \approx 3.1\times 10^{81} \approx 10^{82}.
\]

% \noindent The magnitude of $S$ makes an exhaustive or near-exhaustive search infeasible and motivates alternative approaches that exploit the problem's structure. We note that this approximation does not adjust for overlapping edges between multiple routes.
\noindent Both counts are astronomically large; the transit-center-constrained space is the relevant one for our reported experiments, while the random initialization count describes the larger unconstrained initialization setting. The magnitude of these spaces makes exhaustive or near-exhaustive search infeasible, motivating alternative approaches that exploit the problem's structure. Note that this approximation does not adjust for overlapping edges between routes.

%%%%%%%%%%%%%%%%%%%%%%%%%%%%%%%%%%
% THIS NEEDS TO BE FAIR FOR ALL BASELINES> CANT LET RL JUST WIN BY INCREASING THIS.
%%%%%% NEW (UPDATED)

%% file: appendix_sections/frequency_assignment.tex
\section{Frequency of Service assignment}
\label{app:frequency}

Frequency of Service (FOS) setting is integral to the Transit Route Network Design Problem because it affects both passenger-facing performance, through waiting and in-vehicle movement times, and operator-side cost, through fleet requirements. In a sequential route-construction MDP, however, frequency decisions made for partial routes can provide unstable training targets. A two-node partial route may appear to benefit from high frequency because it serves a small local demand pair, while the same frequency may be inefficient once the route is extended and interacts with the rest of the network. We therefore assign frequencies after route construction using a deterministic max-load projection, following standard capacity-based frequency-setting principles~\cite{ceder2016public,furth1981setting}.

For each route $k$, let $Q^{\mathrm{norm}}_{k,e}(\Pi)$ denote the overlap-normalized segment passenger load rate, in passengers per hour, assigned to segment $e$ of route $k$, and define
\[
Q^{\mathrm{norm}}_{k,\max}(\Pi)
=
\max_{e \in r_k} Q^{\mathrm{norm}}_{k,e}(\Pi).
\]
For this pre-simulation projection, we build an undirected, unweighted route graph from the completed route segments. Each served OD pair contributes $\alpha D_{ij}$ trips per hour to a deterministic minimum-hop path on this graph; direct trips are paths contained on one route, while transfer trips switch routes at shared stops. Segment passenger load rates are accumulated along the selected path and then divided by the number of overlapping routes serving that segment, preventing overestimation of frequency requirements when multiple routes share the same road segment. UXsim then simulates passenger assignment, waiting, boarding, transfers, and traffic dynamics with the resulting frequencies fixed; the projection is not recomputed from realized simulated loads. 

The max-load frequency projection is
\begin{equation}
F^{\mathrm{ml}}_k(\Pi)
=
\max\left\{
1,
\left\lceil
\frac{Q^{\mathrm{norm}}_{k,\max}(\Pi)}
{\delta_{\max} C_k}
\right\rceil
\right\},
\label{eq:max_load_frequency}
\end{equation}
where $C_k$ is bus capacity and $\delta_{\max}$ is the maximum desired load factor. The paper's deterministic frequency vector is therefore $\boldsymbol{\mathcal{F}}(\Pi)=F^{\mathrm{ml}}(\Pi)$. Although the agent does not choose frequencies directly, it influences them through route construction: routes serving high-demand corridors receive higher frequencies, while overlapping routes split normalized segment load and may therefore receive lower frequencies.

\paragraph{Minimality of the max-load projection.}
For a fixed route set $\Pi$, define the capacity-feasible frequency set under fixed normalized segment loads as
\[
\mathcal{F}_{\mathrm{cap}}(\Pi)
=
\left\{
F \in \mathbb{N}_{>0}^{K} :
Q^{\mathrm{norm}}_{k,e}(\Pi)
\le
\delta_{\max} C_k F_k
\quad
\forall k,\; e \in r_k
\right\}.
\]

\begin{proposition}[Minimal fixed-load frequency projection]
\label{prop:minimal_fos_projection}
Fix a route set $\Pi$ and normalized segment loads $Q^{\mathrm{norm}}_{k,e}(\Pi)$ that do not change with the frequency vector $F$. Assume $C_k>0$ for every route $k$ and $\delta_{\max}>0$. Then the max-load rule $F^{\mathrm{ml}}(\Pi)$ is the componentwise minimal element of $\mathcal{F}_{\mathrm{cap}}(\Pi)$. Consequently, among all frequencies in $\mathcal{F}_{\mathrm{cap}}(\Pi)$, it minimizes any operator-frequency cost $C_{\mathrm{op}}(F)$ that is componentwise nondecreasing in $F$.
\end{proposition}

\begin{proof}
For route $k$, feasibility requires
\[
F_k
\ge
\frac{Q^{\mathrm{norm}}_{k,e}(\Pi)}
{\delta_{\max} C_k}
\qquad
\forall e \in r_k.
\]
Therefore,
\[
F_k
\ge
\frac{Q^{\mathrm{norm}}_{k,\max}(\Pi)}
{\delta_{\max} C_k}.
\]
Since $F_k$ must be a positive integer, the smallest feasible value is
\[
\max\left\{
1,
\left\lceil
\frac{Q^{\mathrm{norm}}_{k,\max}(\Pi)}
{\delta_{\max} C_k}
\right\rceil
\right\}
=
F^{\mathrm{ml}}_k(\Pi).
\]
The constraints separate by route, so the same argument holds for every $k \in \{1,\ldots,K\}$. Hence any feasible $F \in \mathcal{F}_{\mathrm{cap}}(\Pi)$ satisfies $F \ge F^{\mathrm{ml}}(\Pi)$ componentwise. Any operator-frequency cost that is componentwise nondecreasing in $F$ is therefore minimized by $F^{\mathrm{ml}}(\Pi)$.
\end{proof}

For fixed route geometry and dispatch assumptions, the fleet requirement used in Eq.~\ref{eq:reward} is componentwise nondecreasing in route frequency, so this minimality property applies to the operator-frequency component of the simulator reward.

Proposition~\ref{prop:minimal_fos_projection} gives the max-load rule a precise role: for the segment loads induced by a completed route set, it is the smallest frequency vector that satisfies the fixed-load capacity condition. The full simulator reward can still depend on frequency through transfer waiting time, passenger assignment, and induced load patterns. We therefore define the value of joint frequency choices relative to this projection.

Because the route count, route length, and admissible operating choices are bounded in our experimental setting, the feasible route and frequency sets are finite. In unbounded variants, the maxima below can be replaced by suprema without changing the interpretation of the projection gap.

\paragraph{Frequency-projection gap.}
Let $\mathcal{P}$ be the feasible set of route sets. For each $\Pi \in \mathcal{P}$, let $\Omega_F(\Pi)$ denote a finite admissible set of positive integer frequency vectors for $\Pi$ that contains $F^{\mathrm{ml}}(\Pi)$. Let
\[
J(\Pi,F)
=
\mathbb{E}_{\xi}
\left[
\mathcal{R}(\Pi,\mathbf{1},F;\xi)
\right]
\]
denote the expected simulator reward for route set $\Pi$, fixed stop spacing $\mathbf{1}$, and frequency vector $F$. Define the joint route-frequency optimum
\[
J^{\star}_{\mathrm{joint}}
=
\max_{\Pi \in \mathcal{P}}
\max_{F \in \Omega_F(\Pi)}
J(\Pi,F),
\]
and the projected route-only optimum optimized in this work:
\[
J^{\star}_{\mathrm{proj}}
=
\max_{\Pi \in \mathcal{P}}
J(\Pi,F^{\mathrm{ml}}(\Pi)).
\]
The frequency-projection gap is
\[
\Delta_F
=
\max_{\Pi \in \mathcal{P}}
\left[
\max_{F \in \Omega_F(\Pi)} J(\Pi,F)
-
J(\Pi,F^{\mathrm{ml}}(\Pi))
\right].
\]

\begin{proposition}[Frequency-projection bound]
\label{prop:frequency_projection_gap}
The value difference between joint route-frequency optimization and the projected route-only problem satisfies
\[
0
\le
J^{\star}_{\mathrm{joint}}
-
J^{\star}_{\mathrm{proj}}
\le
\Delta_F.
\]
\end{proposition}

\begin{proof}
Because $F^{\mathrm{ml}}(\Pi) \in \Omega_F(\Pi)$, the joint optimizer can always choose the projected frequency vector. Therefore,
\[
J^{\star}_{\mathrm{joint}}
=
\max_{\Pi \in \mathcal{P}}
\max_{F \in \Omega_F(\Pi)}
J(\Pi,F)
\ge
\max_{\Pi \in \mathcal{P}}
J(\Pi,F^{\mathrm{ml}}(\Pi))
=
J^{\star}_{\mathrm{proj}}.
\]
This proves the lower bound.

For the upper bound, let $\Pi^{\star}$ be a route set attaining $J^{\star}_{\mathrm{joint}}$. By definition of $\Delta_F$,
\[
\max_{F \in \Omega_F(\Pi^{\star})} J(\Pi^{\star},F)
\le
J(\Pi^{\star},F^{\mathrm{ml}}(\Pi^{\star}))
+
\Delta_F.
\]
Also,
\[
J(\Pi^{\star},F^{\mathrm{ml}}(\Pi^{\star}))
\le
\max_{\Pi \in \mathcal{P}}
J(\Pi,F^{\mathrm{ml}}(\Pi))
=
J^{\star}_{\mathrm{proj}}.
\]
Combining the two inequalities gives
\[
J^{\star}_{\mathrm{joint}}
\le
J^{\star}_{\mathrm{proj}}
+
\Delta_F,
\]
which proves the result.
\end{proof}

The gap $\Delta_F$ identifies the value available from frequency choices after route geometry has been fixed. For a unit vector $e_k$, deliberately increasing service on route $k$ improves a completed route set $\Pi$ when there exists an integer increment $q > 0$ such that
\[
F^{\mathrm{ml}}(\Pi) + q e_k \in \Omega_F(\Pi)
\quad\text{and}\quad
J\bigl(\Pi,F^{\mathrm{ml}}(\Pi)+q e_k\bigr)
>
J\bigl(\Pi,F^{\mathrm{ml}}(\Pi)\bigr).
\]
This includes cases where a connector route has low direct demand but higher frequency improves transfer paths or changes passenger assignment. In the present study, such effects are summarized by $\Delta_F$; extending \systemname{} with a second-stage frequency policy would directly target this gap.

% \[
% F^{\mathrm{ml}}(\Pi) + q e_k \in \Omega_F(\Pi)
% \]
% and
% \[
% J(\Pi,F^{\mathrm{ml}}(\Pi)+q e_k)
% >
% J(\Pi,F^{\mathrm{ml}}(\Pi)).
% \]

%% file: appendix_sections/state_representation.tex
\section{State Representation}
\label{app:state}
% Both \systemname{} and End-to-end RL use the same state representation. 
The state ($s_t$) encodes the static network and evolving design context on $G$ with $s_t \;=\; \bigl(X_t,\ \mathcal{I},\ Z\bigr).$ 

\textbf{Node features ($X_t\in\mathbb{R}^{n\times 16}$).}\; Let $V_{\mathrm{cur}}$ be the nodes already placed on the route under construction at time $t$, $V_{\mathrm{cmp}}$ the union of nodes across completed routes, and $V_{\mathrm{core}}=V_{\mathrm{cur}}\cup V_{\mathrm{cmp}}$. Let $\mathcal{C}_t$ be the set of admissible candidates at time $t$, i.e., the one-hop neighbors of the current frontier that are not already in the route $V_{\mathrm{cur}}$. The node features for node $i$ can then be classified into six groups:

% The Genetic Algorithm covers $111$ nodes ($77.62\%$) with $111.99$~km and achieves an $81.05\%$ service rate.
% Let $\mathcal{C}_t$ be the one hop expansion set formed by neighbors of the current frontier that are not in $V_{\mathrm{cur}}$

\begin{itemize}[leftmargin=12pt]
    \setlength{\itemsep}{0pt}
    \setlength{\parskip}{0pt}
    \item Geometry and connectivity: normalized $(x_i,y_i)$ coordinates and node degree.
    \item OD marginals:
    % \begin{align}
    %   d_{\text{out}}(i) &= \sum_{j} D_{ij}, & 
    %   d_{\text{in}}(i)  &= \sum_{j} D_{ji}. \label{eq:od_marginals}
    % \end{align}
    \begin{align}
      d_{\text{out}}(i) &= \sum_{j=1}^{n} D_{ij}, & 
      d_{\text{in}}(i)  &= \sum_{j=1}^{n} D_{ji}. \label{eq:od_marginals}
    \end{align}
    \item Candidate demand between $i$ and the current route (nonzero only if $i\in\mathcal{C}_t$):
    \begin{align}
      a^{\text{cand}}_{i\to \mathrm{cur}} &= \mathbf{1}\{i\in\mathcal{C}_t\}\sum_{j\in V_{\mathrm{cur}}} D_{ij}, \label{eq:cand_out}\\
      a^{\text{cand}}_{i\leftarrow \mathrm{cur}} &= \mathbf{1}\{i\in\mathcal{C}_t\}\sum_{j\in V_{\mathrm{cur}}} D_{ji}. \label{eq:cand_in}
    \end{align}
    \item Designed network demand at $i$ with respect to nodes already in any designed route (nonzero only if $i\in V_{\mathrm{core}}$):
    \begin{align}
      a^{\text{core}}_{i\to \mathrm{core}} &= \mathbf{1}\{i\in V_{\mathrm{core}}\}\sum_{j\in V_{\mathrm{core}}} D_{ij}, \label{eq:core_out}\\
      a^{\text{core}}_{i\leftarrow \mathrm{core}} &= \mathbf{1}\{i\in V_{\mathrm{core}}\}\sum_{j\in V_{\mathrm{core}}} D_{ji}. \label{eq:core_in}
    \end{align}
    \item Route conditioned demand for all nodes:
    \begin{align}
      a^{\text{all}}_{i\to \mathrm{cur}} &= \sum_{j\in V_{\mathrm{cur}}} D_{ij},&
      a^{\text{all}}_{i\leftarrow \mathrm{cur}} &= \sum_{j\in V_{\mathrm{cur}}} D_{ji}, \label{eq:all_cur}\\
      a^{\text{all}}_{i\to \mathrm{cmp}} &= \sum_{j\in V_{\mathrm{cmp}}} D_{ij},&
      a^{\text{all}}_{i\leftarrow \mathrm{cmp}} &= \sum_{j\in V_{\mathrm{cmp}}} D_{ji}. \label{eq:all_cmp}
    \end{align}
    \item Flags: one indicator each for current route membership $\mathbf{1}\{i\in V_{\mathrm{cur}}\}$, the fraction of completed routes that contain $i$ in $[0,1]$, and valid next node $\mathbf{1}\{i\in\mathcal{C}_t\}$.
\end{itemize}

% trip production and attraction at node $i$ i.e.,

\noindent This representation exposes the policy to demand patterns at multiple scales, from immediate next candidate nodes to the global network, enabling it to reason about ridership potential, transfers, and coverage, without requiring direct access to the OD matrix. Equations \eqref{eq:od_marginals} summarize the baseline incoming and outgoing demands at node $i$, independent of the current design. Equations \eqref{eq:cand_out} and \eqref{eq:cand_in} quantify the immediate marginal gain of selecting a candidate by measuring flows between a candidate and the current route. Equations \eqref{eq:core_out} and \eqref{eq:core_in} summarize how well a node is integrated with the already designed network. The route conditioned terms \eqref{eq:all_cur} and \eqref{eq:all_cmp} provide additional context, i.e., a global view of demand relative to the current and completed routes.
% , making potential transfer opportunities and overlaps explicit in the input

% \vspace{-6pt}
\textbf{Edge connectivity ($\mathcal{I}$).}\; The directed edge list encodes graph topology for message passing. Since streets are bidirectional, we include both $(u,v)$ and $(v,u)$, i.e., $\mathcal{I}\;=\;\{(u,v)\in V\times V \mid \{u,v\}\in E\}$.

% \begin{equation}
%     \mathcal{I}\;=\;\{(u,v)\in V\times V \mid \{u,v\}\in E\}
% \end{equation}
% \vspace{-6pt}

\textbf{Edge features ($Z\in\mathbb{R}^{|\mathcal{I}|\times 2}$).}\;
For each directed edge, $Z$ stores length and free flow speed.

\vspace{4pt}
\noindent All continuous features are min-max scaled to $[0,1]$ using network level bounds. Demand aggregates are computed from $D$ and scaled by the global reference $\max\!\bigl(\max_i\sum_j D_{ij},\ \max_j\sum_i D_{ji}\bigr)$. Terms that depend on $\mathcal{C}_t$ or $V_{\mathrm{core}}$ are set to zero outside those sets. Binary indicators take values in $\{0,1\}$.

%% file: appendix_sections/policy_network.tex
\section{Policy Network}
\label{app:policy}

This section provides additional architecture details for the policy-value network. The actor is permutation equivariant because it applies the same scoring function to every node, while the critic is permutation invariant because it pools node embeddings before predicting a graph-level value. For fixed hidden widths, attention heads, and block count $B$, the number of trainable parameters is independent of the number of nodes $n$.

\textbf{Shared GATv2 backbone.}\;
At state $s=(X,\mathcal{I},Z)$, the node features $X\in\mathbb{R}^{n\times d_x}$ contain per-node spatial attributes, OD-demand aggregates, route-membership indicators, and valid-next-node indicators. The directed edge list $\mathcal{I}$ gives graph connectivity, and $Z\in\mathbb{R}^{|\mathcal{I}|\times 2}$ contains edge attributes, namely link length and free-flow speed. 

Node features are linearly projected to $64$ channels. A stack of $B$ GATv2~\cite{brody2021attentive} blocks produces intermediate node representations $\mathbf{h}^{(b)}\in\mathbb{R}^{n\times d_b}$ for $b\in\{1,\ldots,B\}$. Compared with GAT~\cite{velivckovic2017graph}, GATv2 uses dynamic attention that depends jointly on both endpoints. Each block conditions message passing on edge attributes, so link length and free-flow speed can affect node updates.

Each block applies pre-layer normalization, a GATv2 attention layer, a nonlinear activation, feature dropout, and a residual path. The residual path is the identity when input and output widths match, and a learned linear projection otherwise. Attention dropout is applied inside the GATv2 layer. In the trained configuration, dropout modules are present but disabled.

Jumping Knowledge aggregation~\cite{xu2018jk} preserves information from all message-passing depths. For each node $v$, the outputs of all $B$ blocks are concatenated and projected to the final node embedding:
\[
\mathbf{z}_v = W_{\mathrm{JK}} \left[ \mathbf{h}_v^{(1)}\,\Vert\,\mathbf{h}_v^{(2)}\,\Vert\,\cdots\,\Vert\, \mathbf{h}_v^{(B)} \right] +\mathbf{b}_{\mathrm{JK}}, \qquad H=[\mathbf{z}_1;\ldots;\mathbf{z}_n]\in\mathbb{R}^{n\times d},
\]
where $W_{\mathrm{JK}}\in\mathbb{R}^{d\times\sum_{b=1}^{B}d_b}$ is the Jumping Knowledge projection. For the default $B=4$ configuration, the block widths are $[128,128,64,64]$, so JK aggregation concatenates $128+128+64+64=384$ channels and projects them to $d=64$. This adds $384\cdot64+64=24{,}640$ parameters.

The default attention-head counts are $[8,8,4,4]$. Head outputs are averaged rather than concatenated, so each block output width remains $d_b$. Block $1$ maps $64$ channels to $128$ and uses a projected residual path; Block $2$ keeps width $128$ and uses an identity residual; Block $3$ maps $128$ channels to $64$ and uses a projected residual path; Block $4$ keeps width $64$ and uses an identity residual.

\textbf{Actor head.}\;
The actor uses a pointer-style scoring mechanism~\cite{vinyals2015pointer,yang2022graph}. A shared MLP scores every node embedding $\mathbf{z}_v$ and produces one logit per node. The MLP is applied to all nodes, so the actor is permutation equivariant. A feasibility mask removes inadmissible actions, and the remaining logits are normalized to obtain $P_\theta(\cdot\mid s)$. The actor selects one node per construction step, matching the sequential route-extension process. The default actor MLP has hidden widths $256$, $128$, and $64$.

\textbf{Critic head.}\;
The critic applies global mean pooling and global max pooling to the node embedding matrix $H$, concatenates the two pooled vectors, and passes the result through an MLP to predict $V_\theta(s)$. Pooling makes the critic permutation invariant and produces one scalar value per graph. The default critic MLP has hidden widths $256$, $128$, and $64$.

\textbf{Complexity and implementation.}\;
The backbone is computed once per state to obtain $H$. The actor scores all nodes with a shared MLP, and the critic applies parameter-free global pooling before predicting $V_\theta(s)$. For fixed hidden widths, attention heads, and block count $B$, the trainable parameter count does not depend on $n$. The forward-pass cost grows with the graph as $O(B(n+|\mathcal{I}|))$, which is linear in $n$ for sparse road graphs where $|\mathcal{I}|=O(n)$.

The input contains $16$ node features and $2$ edge features. The default configuration uses $B=4$ GATv2 blocks, final node embedding dimension $d=64$, block widths $[128,128,64,64]$, and attention-head counts $[8,8,4,4]$. More generally, for $B$ blocks we use the channel schedule $[128]^{\lfloor B/2\rfloor}+[64]^{\lceil B/2\rceil}$ and the head schedule $[8]^{\lfloor B/2\rfloor}+[4]^{\lceil B/2\rceil}$. All linear layers use orthogonal initialization, and LayerNorm parameters are initialized to unit scale and zero bias.

%% file: appendix_sections/training_procedures_hyperparameters.tex
\begin{table*}[t!]
\centering
\footnotesize
\renewcommand{\arraystretch}{1.15}
\setlength{\tabcolsep}{3pt}
\resizebox{0.99\textwidth}{!}{%
\begin{tabular}{@{}ll ll ll lll@{}}
\midrule
\multicolumn{2}{@{}l}{\textbf{Simulation}} & \multicolumn{2}{l}{\textbf{Policy Network}} & \multicolumn{2}{l}{\textbf{\systemname{}}} & \multicolumn{3}{l@{}}{\textbf{End-to-End RL}} \\
& & & & & & & $\alpha{=}0.3$ & $\alpha{=}1.0$ \\
\midrule
Horizon & $10{,}000$ & Node features & $16$ & Env steps $S_{\max}$ & $\approx 10^6$ & Env steps & \multicolumn{2}{l}{$10^6$} \\
Time step & $1$\,s & Edge features & $2$ & MCTS sims $N_{\mathrm{iter}}$ & $500$ & Discount $\gamma$ & \multicolumn{2}{l}{$0.999$} \\
Bus capacity & $40$ & GATv2 blocks & $4$ & PUCT $c$ & $1.0$, $1.5$ & GAE $\lambda$ & \multicolumn{2}{l}{$0.95$} \\
Stop duration & $60$\,s & Activation & tanh & Dirichlet $\alpha_{\mathrm{dir}}$ & $0.3$ & Value coef & \multicolumn{2}{l}{$0.5$} \\
Routes $K$ & $16$ & Channels & $128$, $128$, $64$, $64$ & Dirichlet $\varepsilon$ & $0.25$ & Learning rate $\eta$ & $5{\times}10^{-5}$ & $10^{-5}$ \\
Max length $L_{\max}$ & $14$ & Attn heads & $8$, $8$, $4$, $4$ & Buffer size & $50$k & Clip $\epsilon$ & $0.2$ & $0.1$ \\
Modal split $\alpha$ & $0.3$, $1.0$ & Actor MLP & $256$, $128$, $64$ & Workers $W$ & $8$ / $16$ & Epochs $K_e$ & $8$ & $4$ \\
Access radius & $0.5$\,km & Critic MLP & $256$, $128$, $64$ & Batch size & $256$ & Batch size & $256$ & $128$ \\
Stop spacing & $1$ & Attn dropout & $0$, $0$, $0$, $0$ & Learning rate $\eta$ & $10^{-4}$ & Entropy coef & $0.01$ & $0.02$ \\
Platoon $\Delta n$ & $5$ & Hidden dim & $64$ & Train steps/iter & $200$ & LR anneal & No & Yes \\
\midrule
\end{tabular}
}
\vspace{2pt}
\caption{Hyperparameters for simulation, search, and training. Simulation settings are shared; policy settings apply to \systemname{} and End-to-End RL. Pure MCTS lacks policy training and shares search budget and PUCT rule. \systemname{} and Pure MCTS use $N_{\mathrm{iter}}=500$ simulations per decision in final comparisons. Workers are $8$ or $16$. Sweeps selected training hyperparameters; grid sweeps varied search depth, data diversity, and model size. $\tau$ follows $1.0 \to 0.7 \to 0.5$.}
\label{table:hyperparams}
\vspace{-8pt}
\end{table*}

% \caption{Hyperparameters for simulation, policy network, search, and training. Simulation settings are shared across methods, and policy-network settings apply to \systemname{} and End-to-End RL. Pure MCTS has no policy-network training phase and shares only the search budget and PUCT rule. \systemname{} and Pure MCTS use $N_{\mathrm{iter}}=500$ simulations per decision in the final comparisons. Worker counts are $8$ or $16$ depending on the run configuration. Training hyperparameters were selected through sweeps, while search depth, data diversity, and model size were varied with grid-style sweeps. Temperature $\tau$ follows the schedule $1.0 \to 0.7 \to 0.5$ over training.}

\section{Training Procedures and Hyperparameters}
\label{app:training}

Experiments were conducted on an AMD Ryzen Threadripper 7960X CPU with $377$\,GB RAM and an NVIDIA RTX PRO $6000$ GPU. For the $1$M-step runs, End-to-End RL with PPO takes roughly $3$--$5$ hours depending on $\alpha$, \systemname{} with $N_{\mathrm{iter}}=500$ takes roughly $24$--$27$ hours, and Pure MCTS takes roughly $120$--$168$ hours. Table~\ref{table:hyperparams} summarizes the simulation, search, and training settings.

\subsection{Genetic Algorithm}
\label{app:ga}

We implement a genetic algorithm for transit network design where each individual represents a complete transit network $\Pi = (r_1,\ldots,r_K)$ of $K$ routes, each satisfying $L_{\min}\le |r_k|\le L_{\max}$ with $L_{\min}=2$ and $L_{\max}=14$.

\vspace{-8pt}
\begin{itemize}[leftmargin=12pt]
    \setlength{\itemsep}{0pt}
    \setlength{\parskip}{0pt}
    \item Initialization: To prevent premature convergence to local optima, the initial population combines three seeding strategies: (i)~demand-guided construction that greedily extends routes toward high OD-interaction nodes, (ii)~random feasible routes via constrained random walks, and (iii)~a warm-start individual sampled from the existing real-world routes.

    \item Operators: Selection uses tournament selection with size $3$. Crossover performs route exchange where each route index inherits from one parent with equal probability, preserving complete routes to maintain feasibility. Mutation applies path regeneration by cutting a route at a random interior node and regrowing via random walk, preserving promising prefixes while exploring new suffixes. A repair step ensures all routes contain at least $L_{\min}=2$ nodes after each operation, preventing degenerate one-node routes.

    \item Fitness: Each individual is evaluated by assigning frequencies via the max-load rule (Appendix~\ref{app:frequency}) and running simulation. Fitness uses the same terminal-only reward $\mathcal{R}$ (Eq.~\ref{eq:reward}) as \systemname.
\end{itemize}
\vspace{-8pt}

We use population size $50$ with crossover rate $80\%$, mutation rate $40\%$, and elitism count $5$, yielding approximately $5{,}000$ simulator evaluations (approximately $1$M environment steps) and a total wall-clock runtime of $8$ hours. We run $100$ generations per demand setting, with the best design found at generation $94$ for $\alpha=0.3$ and generation $81$ for $\alpha=1.0$.

\subsection{Bee Colony Optimization}
We adapt the implementation from Holliday et al.~\cite{holliday2025learning}, which itself extends the original Bee Colony algorithm of Nikoli\'{c} and Teodorovi\'{c}~\cite{nikolic2013transit}. Two mutation operators act on the route-set population: first, replacement mutation substitutes a dropped route with a shortest-path route between randomly selected terminals; second, endpoint mutation extends or shortens an existing route by adding or removing a node at either endpoint. We use population size $B=15$ with $E=10$ mutations per iteration over $400$ iterations. To match our setup, the implementation is modified so that (i) initialization places every initial route at the transit center, (ii) replacement mutations are forced to originate there, and (iii) endpoint mutations are reverted if they would remove the transit center from a route's start. Because Bee Colony Optimization optimizes the analytical cost function $C = C_p + C_o$, which has no notion of modal split, a single route set is generated and then evaluated through the UXsim simulator at both $\alpha=0.3$ and $\alpha=1.0$. The demand matrix is symmetrized as $D'_{ij} = (D_{ij} + D_{ji})/2$ to satisfy the symmetric demand assumption of the underlying network design problem formulation.

\subsection{Neural Evolutionary Algorithm}
The neural evolutionary algorithm augments Bee Colony Optimization by replacing half of the replacement mutations with route construction proposed by a graph neural network policy~\cite{holliday2024neural,holliday2025learning}. The construction policy is trained from scratch on $32{,}768$ synthetic $20$-node cities following the PPO protocol from Holliday et al.~\cite{holliday2025learning}, with the transit-center start constraint enforced at the model level: a mask tensor identifies the designated start node in each city, and the policy assigns $-\infty$ logits to all other start positions so that only routes originating at that node can be produced. We retrain rather than reuse the published pre-trained weights because those weights were learned without this constraint, causing every neural mutation in our setup to be rejected and the algorithm to degrade to plain Bee Colony Optimization. The algorithm is initialized from the best of $100$ samples drawn from the trained construction policy, then run for $400$ iterations using the same population as Bee Colony Optimization. As with Bee Colony Optimization, a single route set is generated under the analytical objective and then evaluated through UXsim at both demand levels.

\subsection{Pure Monte Carlo Tree Search}
Monte Carlo Tree Search with uniform action priors $P(s,a) = 1/|\mathcal{C}_t|$ over admissible actions and a random rollout till terminal state followed by full route simulation for leaf evaluation. The PUCT selection rule from the main-text Eq.~\ref{eq:puct}, $N_{\mathrm{iter}}=500$, and matched $c$ values per $\alpha$ are the same as in \systemname{}, and a single tree is carried across route boundaries via re-rooting. Each rollout costs $\sim 8$\,s on Bloomington compared to $\sim 1$\,ms for the GNN value head; rollouts are parallelized across $8$ workers, each running an independent simulator instance.

\begin{algorithm}[t!]
\small
\caption{End-to-End Reinforcement Learning}
\label{alg:ppo_trnd}
\begin{algorithmic}[1]
  \STATE \textbf{Input:} Graph $G=(V,E)$, OD matrix $D$, edge index $\mathcal{I}$, edge features $Z$, routes $K$, max route length $L_{\max}$, episodes per update $M$
  \STATE \textbf{Output:} Policy $\pi_\theta$ yielding routes $\Pi=(r_1,\dots,r_K)$
  \STATE Initialize policy-value network $(\pi_\theta,V_\theta)$ and PPO buffer $B\gets\varnothing$
  \WHILE{environment steps $< S_{\max}$}
      \STATE \textit{// Parallel full-episode collection}
      \FOR{workers $w=1,\ldots,N$ \textbf{in parallel} until $M$ episodes are collected}
          \STATE Reset environment; $\Pi\gets\varnothing$; initialize current route $r_1$ from the transit center
          \WHILE{episode not terminated}
              \STATE $s_t \gets \textsc{FormState}(\text{current route},\text{completed routes},\mathcal{I},Z)$
              \STATE $\mathcal{C}_t\gets$ valid one-hop frontier neighbors not already in current route
              \STATE $m_t\gets\textsc{Mask}(\mathcal{C}_t)$
              \IF{$\mathcal{C}_t=\varnothing$}
                  \STATE $a_t\gets\textsc{NoValidAction}$; \quad $\log \pi_{\theta}(a_t\mid s_t,m_t)\gets 0$
              \ELSE
                  \STATE $a_t\sim\pi_\theta(\cdot\mid s_t,m_t)$ and record $\log \pi_{\theta}(a_t\mid s_t,m_t)$
              \ENDIF
              \STATE Execute $a_t$ in the route-construction environment
              \IF{$a_t$ ends a non-final route}
                  \STATE Add completed route to $\Pi$; initialize next route from the transit center
              \ENDIF
              \IF{$a_t$ ends the final route}
                  \STATE Add final route to $\Pi$; run UXsim once on $\Pi$; set $r_t\gets\mathcal{R}$ (Eq.~\ref{eq:reward})
              \ELSE
                  \STATE Set $r_t\gets\mathcal{R}_{\mathrm{partial}}$ (Eq.~\ref{eq:partial}), or $0$ in terminal-only mode
              \ENDIF
              \STATE Store $(s_t,a_t,r_t,\log\pi_\theta(a_t\mid s_t,m_t),V_\theta(s_t),m_t,\texttt{done})$ locally
          \ENDWHILE
          \STATE Compute GAE advantages $\hat{A}_t$ and returns over the full episode; append trajectory to $B$
      \ENDFOR
      \STATE \textit{// PPO optimization}
      \FOR{epoch $=1$ to $K_{\mathrm{epochs}}$}
          \FOR{minibatch $b\sim B$}
              \STATE $\theta \gets \theta-\eta\nabla_\theta\mathcal{L}_{\mathrm{PPO}}(\theta)$ using Eq.~\ref{eq:ppo_loss}
          \ENDFOR
      \ENDFOR
      \STATE Clear $B$ and broadcast updated weights to workers
  \ENDWHILE
  \STATE \textbf{return} $\pi_\theta$
\end{algorithmic}
\end{algorithm}

% \label{app:endtoend}
\subsection{End-to-End Reinforcement Learning}
A key challenge in applying standard policy gradient methods to the Transit Route Network Design Problem (TRNDP) is the long episode horizon: with $K=16$ routes of up to $L_{\max}=14$ nodes each, episodes can span over $200$ decisions before any terminal reward is available. This creates a severe credit assignment problem~\cite{sutton2018reinforcement}, as the policy must learn which early decisions contributed to the final outcome. To address this, the end-to-end approach augments the reward from Eq.~\eqref{eq:reward} with lightweight shaping rewards during route construction. These provide more frequent feedback to guide learning, while the simulation-based reward remains the primary signal. All reported End-to-End RL results use the same delta-based shaping rule without early-stop penalty at both demand levels. The full training procedure is summarized in Algorithm~\ref{alg:ppo_trnd}, where $V_{\mathrm{cur}}$ denotes nodes in the current route and $V_{\mathrm{cmp}}$ denotes nodes across all completed routes.

Helpers in Algorithm~\ref{alg:ppo_trnd} refer to route-construction environment operations. \textsc{FormState} builds the graph observation from the current and completed routes, \textsc{Mask} converts $\mathcal{C}_t$ into the binary feasibility mask used by the policy, and \textsc{NoValidAction} denotes the forced route-finalization transition when no valid extension remains; it is not a learned stop action.

During the construction of a route $r_k$, the agent receives feedback based on the evolving network topology. We use the incremental change in demand coverage $\Delta\Psi_t=\max(0,\Psi_t-\Psi_{t-1})$ to reward actions that expand the reachable O-D pairs. This delta-based formulation directly measures the value of each action rather than accumulated state, providing clearer credit assignment. The reported End-to-End RL baseline penalizes route overlap but does not include an early-stop term:
\begin{equation}
  \mathcal{R}_{\text{partial}}
  = b_7 \cdot \Delta\Psi_t - b_8 \cdot \omega,
\label{eq:partial}
\end{equation}
where $\omega$ is the current route overlap ratio. We use $b_7=20$ and $b_8=8$. The coverage term is marginal, while the overlap term is a step-wise exposure penalty: during PPO rollout, repeatedly overlapping partial networks accumulate overlap cost across construction steps. This shaping term is used only for the End-to-End RL baseline. \systemname{} trains from terminal full-network evaluations, so its overlap penalty is applied through Eq.~\eqref{eq:reward} once per completed route set. These shaping terms remain small relative to the simulation-based reward so that the primary learning signal still comes from passenger flow outcomes. Alternative shaping variants, including early-stop penalties, are analyzed in Fig.~\ref{fig:learning_overview_0_3}. Partial rewards are also normalized online during training. Forced route termination is treated as an environment transition, and full UXsim evaluation is performed exactly once after all $K$ routes have been finalized.

The PPO update minimizes the negative clipped surrogate with a value loss and entropy bonus,
\begin{equation}
\label{eq:ppo_loss}
\begin{aligned}
\mathcal{L}_{\mathrm{PPO}}(\theta)
&=
-\mathbb{E}_t
\left[
\min\!\left(
\rho_t(\theta)\hat{A}_t,\,
\mathrm{clip}\!\left(\rho_t(\theta),1-\epsilon,1+\epsilon\right)\hat{A}_t
\right)
\right]
 + c_v L_V - c_e H,\\
\rho_t(\theta)
&=
\frac{\pi_\theta(a_t\mid s_t,m_t)}
{\pi_{\theta_{\mathrm{old}}}(a_t\mid s_t,m_t)} .
\end{aligned}
\end{equation}

% Reward coefficients (Eq.\ref{eq:reward}) and shaping weights (Eq.\ref{eq:partial}) were selected via The search explored [X] configurations

\vspace{10pt}
\subsection{\systemname{}}
\label{app:algorithm}
Training alternates between parallel data generation and network optimization. During data generation, $W$ workers independently construct complete transit networks using Monte Carlo Tree Search guided decisions. At each step, Dirichlet noise is added to the root priors to encourage exploration:
\begin{equation}
P(s,a) \leftarrow (1 - \varepsilon) \cdot P(s,a) + \varepsilon \cdot \eta_a, \quad \eta \sim \mathrm{Dir}(\alpha_{\mathrm{dir}}),
\label{eq:dirichlet}
\end{equation}
where $\alpha_{\mathrm{dir}}$ controls noise concentration and $\varepsilon$ controls the noise weight. After ordinary route-extension actions, the search tree is re-rooted at the selected child so the corresponding subtree statistics are reused. When a route has no valid extension and is finalized automatically, we root a new tree at the forced successor state. Algorithm~\ref{alg:alphatransit} summarizes the full \systemname{} training loop. Training runs for $S_{\max}$ environment steps across $W$ parallel workers using learning rate $\eta$.

Helpers in Algorithm~\ref{alg:alphatransit} refer to environment operations. \textsc{InitializeRoute} starts a route at the configured transit hub, \textsc{FormState} builds the graph observation from the environment state, and \textsc{ValidActions}$(s_t)$ returns the admissible one-hop extensions $\mathcal{C}_t$. \textsc{Append} stores each search target tuple for replay, and \textsc{Routes} returns the completed route tuple before terminal simulation.

\begin{algorithm}[t!]
  \caption{AlphaTransit}
  \label{alg:alphatransit}
  \begin{algorithmic}[1]
    \STATE \textbf{Input:} Graph $G$, OD matrix $D$, edges $\mathcal{I}$, features $Z$, routes $K$, max length $L_{\max}$, MCTS simulations $N_{\mathrm{iter}}$, workers $W$.
    \STATE \textbf{Output:} Policy-value network $f_\theta$ yielding route tuples $\Pi=(r_1,\ldots,r_K)$.
    \STATE \textbf{Initialize:} Network $f_\theta$, replay buffer $\mathcal{D}\gets\varnothing$, reward statistics $\mathcal{Z}$.
    \WHILE{$t_{\mathrm{env}} < S_{\max}$}
        \STATE Broadcast current parameters $\theta$ to workers; set temperature $\tau$.
        \STATE \textit{// Parallel MCTS-guided episode collection}
        \FOR{each worker $w\in\{1,\ldots,W\}$ \textbf{in parallel}}
            \STATE Reset environment; $r_1\gets[\textsc{InitializeRoute}()]$; initialize MCTS tree $\mathcal{T}$.
            \STATE $\mathcal{E}_w\gets[\,]$; $H_w\gets0$.
            \WHILE{route set is incomplete}
                \STATE $s_t\gets\textsc{FormState}(\mathcal{I},Z)$; $\mathcal{C}_t\gets\textsc{ValidActions}(s_t)$.
                \IF{$\mathcal{C}_t=\varnothing$}
                    \STATE Finalize current route; initialize the next route if one remains; reset $\mathcal{T}$.
                    \STATE $H_w\gets H_w+1$; \textbf{continue}.
                \ENDIF
                \STATE Run $N_{\mathrm{iter}}$ MCTS simulations from $s_t$ using $f_\theta$ and PUCT.
                \STATE Derive $\pi_t(a\mid s_t)\propto N(s_t,a)^{1/\tau}$ over $a\in\mathcal{C}_t$.
                \STATE $\mathcal{E}_w.\textsc{Append}(s_t,\pi_t,\mathcal{C}_t)$.
                \STATE Sample $a_t\sim\pi_t$; apply $a_t$; re-root $\mathcal{T}$ at the selected child.
                \STATE $H_w\gets H_w+1$.
            \ENDWHILE
            \STATE $\Pi_w\gets\textsc{Routes}()$; $z_w\gets R_T(\Pi_w;\xi)$.
        \ENDFOR
        \STATE \textit{// Aggregate episodes}
        \FOR{$w=1$ to $W$}
            \STATE Update reward statistics $\mathcal{Z}$ with raw reward $z_w$.
            \STATE Add $(s_i,\pi_i,\mathcal{C}_i,z_w)$ to $\mathcal{D}$ for all $(s_i,\pi_i,\mathcal{C}_i)\in\mathcal{E}_w$.
        \ENDFOR
        \STATE $t_{\mathrm{env}}\gets t_{\mathrm{env}}+\sum_w H_w$.
        \STATE \textit{// Policy-value network optimization}
        \FOR{step $=1$ to $N_{\mathrm{steps}}$}
            \STATE Sample minibatch $\mathcal{B}=\{(s_j,\pi_j,\mathcal{C}_j,z_j)\}_{j=1}^{m}\subseteq\mathcal{D}$.
            \STATE Normalize each $z_j$ with $\mathcal{Z}$ to obtain $\tilde{z}_j$.
            \STATE Evaluate $f_\theta(s_j)$ and mask $P_\theta$ over $\mathcal{C}_j$ for each tuple.
            \STATE Define $\ell_j(\theta)=-\sum_{a\in\mathcal{C}_j}\pi_j(a\mid s_j)\log P_\theta(a\mid s_j)+\left(V_\theta(s_j)-\tilde{z}_j\right)^2$.
            \STATE $\theta\gets\theta-\eta\nabla_\theta\tfrac{1}{m}\sum_{j=1}^{m}\ell_j(\theta)$.
        \ENDFOR
    \ENDWHILE
    \STATE \textbf{return} $f_\theta$.
  \end{algorithmic}
\end{algorithm}

\vspace{-4pt}
\begin{itemize}[leftmargin=12pt]
    \setlength{\itemsep}{0pt}
    \setlength{\parskip}{0pt}

    \item Value Normalization: Terminal rewards can vary significantly in scale across episodes. To stabilize learning, we normalize rewards online:
        \[
        \tilde{z} = \mathrm{clip}\left(\frac{z - \mu}{\sigma + \epsilon}, -3, 3\right),
        \]
    where running statistics $(\mu, \sigma)$ are updated with all raw rewards before normalizing.

    \item Temperature Schedule: The temperature $\tau$ in Eq.~\eqref{eq:mcts_policy} uses a schedule based on training progress:
        \[
        \tau(\texttt{progress}) =
        \begin{cases}
        1.0 & \text{if } \texttt{progress} < 0.3, \\
        0.7 & \text{if } 0.3 \le \texttt{progress} < 0.6, \\
        0.5 & \text{otherwise}.
        \end{cases}
        \]
\end{itemize}
\vspace{-8pt}

During evaluation, \systemname{} uses near-greedy selection ($\tau = 0.1$) over the MCTS visit-count distribution and disables Dirichlet noise. For End-to-End RL, we use low-temperature sampling ($\tau = 0.1$) rather than argmax selection. This follows standard practice in neural combinatorial optimization~\cite{kool2019attention, kwon2020pomo, bello2016neural}, where sampling is preferred over greedy selection for three reasons: (i)~the policy is trained to optimize expected return under its stochastic distribution, not the argmax; (ii)~deterministic selection produces identical route sets across evaluation runs, precluding any measure of variance; and (iii)~\systemname{}'s near-greedy evaluation operates over search-refined visit counts, not raw logits, so aligning PPO's evaluation temperature provides a fairer comparison.

% Both methods report mean $\pm$ standard deviation over $10$ evaluation seeds.

%% file: appendix_sections/real_world_networks_data.tex
\section{Real-world Networks and Data}
\label{app:network}

\begin{figure*}[t!]
    \centering
    \includegraphics[width=0.98\linewidth]{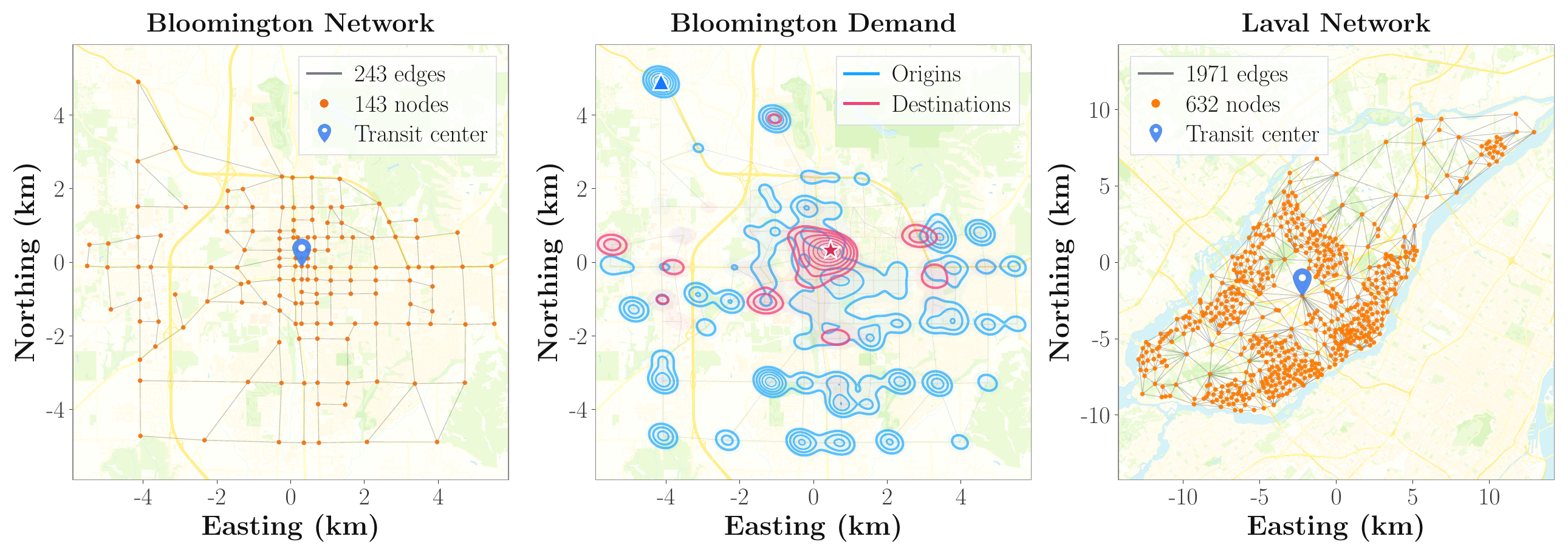}
    \vspace{-8pt}
    \caption{The Bloomington, Indiana transportation network ($\sim$152.3 km$^2$). The demand map shows the spatial distribution of trip origins (blue) and destinations (red), with peak origin demand of $344$ trips per hour at node $1$ ($\triangle$) and peak destination demand of $1{,}681$ trips per hour at node $129$ ($\star$). The Laval, Quebec transportation network ($\sim 256$ km$^2$) is used as the out-of-distribution generalization benchmark. The Laval network, with $632$ nodes and $1{,}971$ links, is obtained from Holliday et al.~\cite{holliday2025learning}; the blue transit-center marker indicates node $542$.}
    \vspace{-10pt}
    \label{fig:bloomington}
\end{figure*}

\subsection{Bloomington, Indiana}
\label{app:bloomington}

We introduce a novel city-scale transit network dataset for Bloomington, Indiana. Unlike existing benchmark networks from literature, our dataset uniquely captures real-world aspects in three dimensions: (i) the underlying transportation network, (ii) travel demand derived from census data, and (iii) transit routes currently operating in the city. 

% The network topology was derived from actual road infrastructure to ensuring accurate representation of real-world connectivity and spatial relationships.
\textbf{Network Structure.}\;
The network consists of $143$ nodes and $243$ bidirectional edges, covering an area of approximately $152.3~\text{km}^{2}$. The network topology was derived from road infrastructure with several practical assumptions made to balance modeling fidelity with simulation efficiency.
\vspace{-8pt}
\begin{itemize}[leftmargin=12pt]
    \setlength{\itemsep}{0pt}
    \setlength{\parskip}{0pt}
    \item Planar Representation: Three-dimensional infrastructure elements such as tunnels, overpasses, and underpasses are modeled as planar connections.
    \item Edge Geometry: All edges are represented as bidirectional. When two parallel one-way streets exist next to each other, they are consolidated into a single bidirectional edge positioned at their centerline. Further, edges follow shortest-distance connections between nodes rather than exact street curvatures; the length differences are negligible for the scale of analysis.
    \item Speed Assignment: All edges are assigned a uniform free-flow speed of $16.67$ m/s ($60$ km/h or $37$ miles/hr), reflecting typical urban traffic speed limit.
    \item Highway Exclusion: Interstate $69$ highway segments within the city limits were excluded from the transit network, as city transit primarily serves local destinations and highways lack appropriate passenger access points. The bus routes currently operating in the city also avoid the highway.
    \item Shared Routes: Although existing bus routes may have slightly different outbound and inbound paths, to reduce complexity, we model them as identical paths, manually selecting the most appropriate nodes (based on access, length, and community served) to maintain essential connectivity.
\end{itemize}

% \vspace{-8pt}
\textbf{Coordinate Transformation.}\;
The raw geospatial data uses geographic coordinates (latitude and longitude) in angular units. However, operations in our processing pipeline such as calculating census block centroids and edge distances require a flat, Cartesian plane for accurate results. Therefore, the coordinates were transformed to a Cartesian coordinate system using the Universal Transverse Mercator (UTM) Zone~$16$N projection, which covers Indiana including Bloomington. The UTM projection preserves local angles and shapes while providing true metric distances and areas with minimal distortion, making it ideal for our case.

% We did not divide by 250. we divide by $250$ (workdays/ year)
% which consists of non-overlapping regions designed to be relatively permanent
\textbf{Demand Generation.}\;
The primary demand component was obtained from commuting trips captured in the $2022$ LEHD Origin-Destination Employment Statistics (LODES) from the U.S. Census Bureau~\cite{uscensus_lehd_lodes}, which provides flows between census blocks with home locations as origins and work locations as destinations. We utilize block-level census data~\cite{uscensus_tiger_shapefiles} and processed a total of $2{,}399$ census blocks within Monroe County, Indiana (FIPS code $105$, GEOID $18105$), which was further reduced to $1{,}475$ blocks to confine the area of interest to the vicinity of Bloomington. For each census block, its centroid was calculated and the demand origins and destinations in that block were assigned to the node nearest to the centroid.

Because the LODES data excludes trips for non-commuting purposes such as school and shopping (typically classified as home-based other or non-home-based in transportation research), to account for this mixed composition of traffic, we scale the commuting flows by $150\%$, consistent with typical values ranging from $100\%-200\%$ depending on time of day~\cite{mcguckin2018summary}. Further, to express demand on an hourly basis, we adopt a peak‑hour share of $11\%$ of daily traffic, which lies within the typical $6\%-12\%$ range~\cite{roess2011traffic}. The resulting origin–destination (OD) demand matrix contains $5{,}737$ pairs, with a maximum origin demand of $344$ trips per hour and a maximum destination demand of $1{,}681$ trips per hour.
This OD matrix is treated as an exogenous peak-hour trip table. In each experiment, $\alpha$ is fixed before evaluation and scales the table into a transit-demand scenario.

\textbf{Existing Bus Routes.}\;
The Bloomington Transit system~\cite{bloomingtontransit_gtfs, transitland_bloomingtontransit} operates $16$ bus routes which map to an average of $14.2$ nodes per route in our network (ranging from $8$ to $24$ nodes). 

\subsection{Laval, Quebec}
\label{app:laval}

\begin{figure}[t!]
    \centering
    \includegraphics[width=0.98\linewidth]{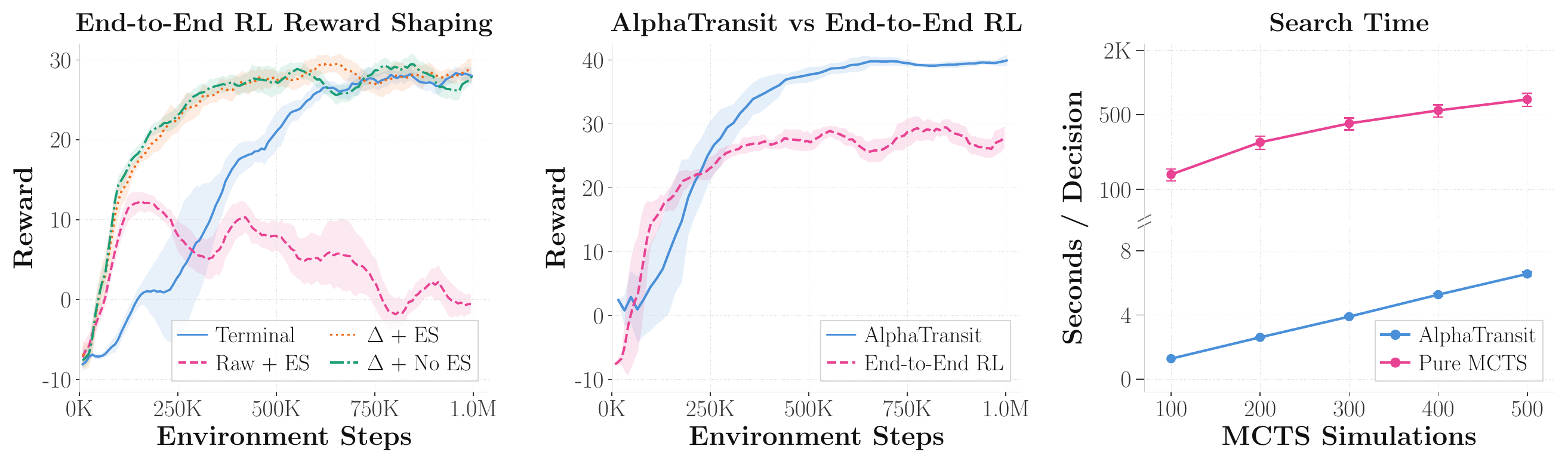}
    \vspace{-6pt}
    \caption{Learning dynamics and search cost under full demand ($\alpha=1.0$). \textbf{LEFT}: Curves average two training seeds per reward mode. Reward shaping remains important for End-to-End RL, although the higher service-weight objective is more forgiving than under mixed demand ($\alpha=0.3$). Early Stopping (ES) penalizes routes that terminate before $L_{\max}$, while Delta-coverage ($\Delta$) rewards only newly covered demand. The $\Delta$ variants give the strongest final performance, with $\Delta$ + ES ending slightly highest. \textbf{MIDDLE}: MCTS-provided search targets improve sample efficiency under the same environment-step budget. At search depth $N_{\mathrm{iter}}=500$, \systemname{} leads from earliest plotted evaluations and finishes with reward $39.93$, compared with $27.92$ for End-to-End RL. \textbf{RIGHT}: Learned policy and value estimates make MCTS search practical across search depths $N_{\mathrm{iter}}\in[100,500]$ on the Bloomington network using one CPU worker. At $N_{\mathrm{iter}}=500$, \systemname{} requires $6.56$ seconds per decision, while Pure MCTS requires $695.48$ seconds per decision.}
    \label{fig:learning_overview_1_0}
    \vspace{-8pt}
\end{figure}

To evaluate cross-city generalization, we use the Laval, Qu\'{e}bec network from Holliday et al.~\cite{holliday2025learning}. The Laval graph contains $632$ nodes and $1{,}971$ bidirectional edges over approximately $256$\,km$^2$, making it roughly $4.4{\times}$ larger than Bloomington in node count. We use node $542$ as the transit center because it has the highest degree of $21$, highest closeness centrality, and highest betweenness centrality in the network. Despite the larger node count, Laval has the same graph diameter of $17$ hops and a comparable mean shortest-path length of $7.1$ hops because of its higher connectivity, with average degree $6.24$ compared with $3.40$ for Bloomington. This structural similarity supports using the same route-design parameters as in Bloomington, namely $K{=}16$ routes and maximum route length $L_{\max}{=}14$. In Laval, $99.7\%$ of node pairs are reachable within $14$ hops, so this length bound remains appropriate despite the larger graph. For simulation, Laval's unscaled demand rate is about $548$K trips per hour, roughly $60{\times}$ Bloomington's, and is uniformly scaled by $0.14{\times}$ to match per-link traffic density, yielding about $77$K trips per hour. The Bloomington-trained agent is evaluated directly on Laval without additional training or fine-tuning. Because the shared Laval dataset does not include real-world bus routes, the Real-World baseline is excluded from this evaluation.

% Policy network figure moved to Section~\ref{sec:policy_network} (Methodology).

%%%%%%%%%%%%%%%%%%%%%%%% MORE EXACT CALCULATION

%% file: appendix_sections/results.tex
\section{Extended Results}
\label{app:results}
\label{app:metrics}

% We use $\tau_p^{\text{wait}}$ and $\tau_p^{\text{move}}$ for passenger-level accumulated waiting and in-vehicle movement times, and barred quantities for averages in minutes.
We evaluate every transit design through the same simulation pipeline. At simulation end, $N_{\text{comp}}$, $N_{\text{ongoing}}$, and $N_{\text{waiting}}$ denote passengers who completed trips, are onboard buses, and are waiting at stops, respectively. We set $N_{\text{boarded}}=N_{\text{served}}=N_{\text{comp}}+N_{\text{ongoing}}$ for passengers counted as served, and $N_{\text{want}}$ denotes potential riders. We use $\tau_p^{\text{wait}}$ and $\tau_p^{\text{move}}$ for passenger-level accumulated waiting and in-vehicle movement times, respectively, and variables with an overbar denote averages in minutes. The seven evaluation metrics are:
\begin{itemize}[leftmargin=12pt]
    \setlength{\itemsep}{0pt}
    \setlength{\parskip}{0pt}
    \item Service rate (\%): Fraction of potential demand counted as served, $\sigma = N_{\text{boarded}} / N_{\text{want}}$. This differs from the fixed-denominator service term $\rho=N_{\mathrm{boarded}}/N_{\mathrm{OD}}$ used by the training reward.

    \item Wait time (min): Average waiting time over served riders, $\bar{t}_{\text{wait}}$.

    \item Transfer rate (\%): Fraction of completed trips requiring at least one transfer, $N_{\text{transfer}} / N_{\text{comp}} \times 100$.

    \item Journey time (min): Average elapsed passenger journey time for passengers counted as served,
    \begin{equation}
        \bar{t}_{\mathrm{journey}}
        =
        \frac{1}{N_{\text{served}}}
        \sum_{p \in \mathcal{S}_{\mathrm{served}}}
        \bigl(\tau_p^{\text{wait}} + \tau_p^{\text{move}}\bigr),
    \end{equation}
    where $\mathcal{S}_{\mathrm{served}}$ contains passengers who completed trips or are onboard at the simulation end. This metric includes both waiting and in-vehicle movement.

    \item Route efficiency (pax/km): Passengers served per unit infrastructure, $N_{\text{comp}} / L_{\text{total}}$, where $L_{\text{total}}$ is the total route length in kilometers.

    \item Fleet size: Total number of buses deployed across all routes, $N_{\mathrm{bus}}$.

    \item Bus utilization (\%): Average load across all buses, $u$, computed as occupancy divided by capacity.
\end{itemize}

\begin{table}[t!]
  \centering
  % \small
  \renewcommand{\arraystretch}{1.25}
  \resizebox{0.99\textwidth}{!}{%
    \begin{tabular}{cl ccccccc}
        & & \multicolumn{4}{c}{Passenger Metrics} & \multicolumn{3}{c}{Operator Metrics} \\
        \cmidrule(lr){3-6} \cmidrule(lr){7-9}
        & Method
        & \makecell{Service \\ Rate (\%) $\uparrow$}
        & \makecell{Wait \\ Time (min) $\downarrow$}
        & \makecell{Transfer \\ Rate (\%) $\downarrow$}
        & \makecell{Journey \\ Time (min) $\downarrow$}
        & \makecell{Route \\ Efficiency $\uparrow$}
        & \makecell{Fleet \\ Size $\downarrow$}
        & \makecell{Bus \\ Util. (\%) $\uparrow$} \\
        \midrule
        \multirow{10}{*}{\rotatebox{90}{$\alpha = 0.3$}}
        & Real-World & $42.77 {\scriptstyle \pm 1.05}$ & $14.02 {\scriptstyle \pm 0.56}$ & $86.05 {\scriptstyle \pm 0.44}$ & $46.83 {\scriptstyle \pm 1.09}$ & $13.15 {\scriptstyle \pm 0.42}$ & $89.00$ & $17.69 {\scriptstyle \pm 0.18}$ \\
        & Random Walk & $36.41 {\scriptstyle \pm 3.68}$ & $21.79 {\scriptstyle \pm 2.75}$ & $82.26 {\scriptstyle \pm 3.39}$ & $49.66 {\scriptstyle \pm 3.52}$ & $5.06 {\scriptstyle \pm 1.27}$ & $35.40 {\scriptstyle \pm 4.54}$ & $11.10 {\scriptstyle \pm 2.10}$ \\
        & Demand Cover & $40.11 {\scriptstyle \pm 3.43}$ & $19.15 {\scriptstyle \pm 3.52}$ & $77.45 {\scriptstyle \pm 3.29}$ & $45.46 {\scriptstyle \pm 4.61}$ & $6.18 {\scriptstyle \pm 1.76}$ & $40.10 {\scriptstyle \pm 10.95}$ & $12.42 {\scriptstyle \pm 2.74}$ \\
        & Shortest Path & $37.89 {\scriptstyle \pm 4.17}$ & $23.35 {\scriptstyle \pm 2.73}$ & $84.43 {\scriptstyle \pm 2.56}$ & $45.06 {\scriptstyle \pm 4.06}$ & $4.75 {\scriptstyle \pm 0.80}$ & $21.60 {\scriptstyle \pm 2.58}$ & $6.53 {\scriptstyle \pm 1.07}$ \\
        & Genetic Alg. & $50.42 {\scriptstyle \pm 0.50}$ & $9.07 {\scriptstyle \pm 0.50}$ & $82.61 {\scriptstyle \pm 0.27}$ & $39.14 {\scriptstyle \pm 0.59}$ & $17.12 {\scriptstyle \pm 0.21}$ & $79.00$ & $19.76 {\scriptstyle \pm 0.21}$ \\
        & Bee Colony & $39.83 {\scriptstyle \pm 0.66}$ & $10.50 {\scriptstyle \pm 0.49}$ & $87.97 {\scriptstyle \pm 0.20}$ & $34.09 {\scriptstyle \pm 0.44}$ & $18.58 {\scriptstyle \pm 0.25}$ & $94.00$ & $12.07 {\scriptstyle \pm 0.09}$ \\
        & Neural Evol. & $47.85 {\scriptstyle \pm 0.57}$ & $5.65 {\scriptstyle \pm 0.20}$ & $92.16 {\scriptstyle \pm 0.11}$ & $29.94 {\scriptstyle \pm 0.29}$ & $23.71 {\scriptstyle \pm 0.22}$ & $101.00$ & $13.11 {\scriptstyle \pm 0.16}$ \\
        & Pure MCTS & $53.30 {\scriptstyle \pm 0.97}$ & $7.42 {\scriptstyle \pm 0.54}$ & $84.97 {\scriptstyle \pm 2.58}$ & $37.57 {\scriptstyle \pm 1.34}$ & $19.87 {\scriptstyle \pm 1.01}$ & $86.00 {\scriptstyle \pm 4.15}$ & $21.93 {\scriptstyle \pm 1.26}$ \\
        & End-to-End RL & $49.72 {\scriptstyle \pm 1.77}$ & $8.46 {\scriptstyle \pm 1.28}$ & $84.49 {\scriptstyle \pm 1.08}$ & $41.14 {\scriptstyle \pm 1.77}$ & $18.09 {\scriptstyle \pm 1.43}$ & $111.90 {\scriptstyle \pm 6.04}$ & $19.89 {\scriptstyle \pm 1.02}$ \\
        & AlphaTransit & $54.64 {\scriptstyle \pm 0.54}$ & $7.13 {\scriptstyle \pm 0.16}$ & $82.66 {\scriptstyle \pm 0.29}$ & $35.81 {\scriptstyle \pm 0.40}$ & $21.99 {\scriptstyle \pm 0.33}$ & $80.00$ & $22.10 {\scriptstyle \pm 0.18}$ \\
        \midrule
        \multirow{10}{*}{\rotatebox{90}{$\alpha = 1.0$}}
        & Real-World & $58.44 {\scriptstyle \pm 0.95}$ & $15.95 {\scriptstyle \pm 0.39}$ & $81.90 {\scriptstyle \pm 0.31}$ & $52.85 {\scriptstyle \pm 0.72}$ & $61.83 {\scriptstyle \pm 0.92}$ & $281.00$ & $31.89 {\scriptstyle \pm 0.25}$ \\
        & Random Walk & $62.79 {\scriptstyle \pm 5.55}$ & $15.72 {\scriptstyle \pm 2.86}$ & $79.87 {\scriptstyle \pm 2.66}$ & $47.45 {\scriptstyle \pm 3.22}$ & $36.85 {\scriptstyle \pm 9.30}$ & $109.40 {\scriptstyle \pm 31.86}$ & $28.61 {\scriptstyle \pm 3.77}$ \\
        & Demand Cover & $58.03 {\scriptstyle \pm 7.49}$ & $16.28 {\scriptstyle \pm 2.61}$ & $75.08 {\scriptstyle \pm 4.23}$ & $47.49 {\scriptstyle \pm 3.26}$ & $35.52 {\scriptstyle \pm 7.52}$ & $124.60 {\scriptstyle \pm 32.86}$ & $26.83 {\scriptstyle \pm 2.83}$ \\
        & Shortest Path & $56.50 {\scriptstyle \pm 12.49}$ & $20.37 {\scriptstyle \pm 4.73}$ & $69.60 {\scriptstyle \pm 8.07}$ & $42.84 {\scriptstyle \pm 2.88}$ & $23.47 {\scriptstyle \pm 8.92}$ & $48.00 {\scriptstyle \pm 12.20}$ & $18.15 {\scriptstyle \pm 3.34}$ \\
        & Genetic Alg. & $81.17 {\scriptstyle \pm 0.99}$ & $8.94 {\scriptstyle \pm 0.45}$ & $87.28 {\scriptstyle \pm 0.23}$ & $44.32 {\scriptstyle \pm 0.57}$ & $100.19 {\scriptstyle \pm 1.87}$ & $254.00$ & $43.16 {\scriptstyle \pm 0.29}$ \\
        & Bee Colony & $64.74 {\scriptstyle \pm 1.49}$ & $11.96 {\scriptstyle \pm 0.60}$ & $84.45 {\scriptstyle \pm 0.20}$ & $48.58 {\scriptstyle \pm 0.67}$ & $89.36 {\scriptstyle \pm 0.98}$ & $301.00$ & $28.78 {\scriptstyle \pm 0.48}$ \\
        & Neural Evol. & $70.51 {\scriptstyle \pm 1.85}$ & $11.91 {\scriptstyle \pm 0.91}$ & $87.34 {\scriptstyle \pm 0.23}$ & $49.20 {\scriptstyle \pm 0.51}$ & $99.64 {\scriptstyle \pm 1.42}$ & $320.00$ & $30.39 {\scriptstyle \pm 0.49}$ \\
        & Pure MCTS & $73.79 {\scriptstyle \pm 4.96}$ & $8.68 {\scriptstyle \pm 1.91}$ & $80.07 {\scriptstyle \pm 2.66}$ & $43.92 {\scriptstyle \pm 3.22}$ & $77.39 {\scriptstyle \pm 5.41}$ & $200.80 {\scriptstyle \pm 24.75}$ & $36.29 {\scriptstyle \pm 1.81}$ \\
        & End-to-End RL & $73.70 {\scriptstyle \pm 1.68}$ & $10.04 {\scriptstyle \pm 0.82}$ & $85.29 {\scriptstyle \pm 0.53}$ & $50.52 {\scriptstyle \pm 1.30}$ & $83.78 {\scriptstyle \pm 4.40}$ & $346.50 {\scriptstyle \pm 14.54}$ & $43.31 {\scriptstyle \pm 1.57}$ \\
        & AlphaTransit & $82.08 {\scriptstyle \pm 0.55}$ & $8.48 {\scriptstyle \pm 0.39}$ & $82.55 {\scriptstyle \pm 0.19}$ & $43.10 {\scriptstyle \pm 0.49}$ & $110.58 {\scriptstyle \pm 0.90}$ & $267.00$ & $45.02 {\scriptstyle \pm 0.24}$ \\
        \bottomrule
    \end{tabular}}
    \vspace{8pt}
    \caption{Performance comparison across ten methods under mixed demand ($\alpha=0.3$) and full transit demand ($\alpha=1.0$) on the Bloomington benchmark. \systemname{} achieves the highest service rate in both demand settings, with $54.64\%$ under mixed demand and $82.08\%$ under full transit demand; under full transit demand, it also achieves the best wait time, route efficiency, and bus utilization. Arrows indicate the direction of improvement ($\uparrow$ higher is better, $\downarrow$ lower is better), and values report mean $\pm$ standard deviation over $10$ evaluation seeds for each trained policy or generated route set. \systemname{} uses $N_{\mathrm{iter}}=500$, while End-to-End RL is the PPO baseline without MCTS. Fleet-size deviations appear only when the route-design procedure yields different route sets across seeds, since frequencies and fleet size are deterministic once a route set is fixed.}
   \label{tab:modal_split_results_extended}
   \vspace{-8pt}
\end{table}

The demand coverage potential $\Psi$ is defined in Section~\ref{sec:method}. The route overlap ratio $\omega$ is computed over unique undirected road segments. Let $\mathcal{E}_{\Pi}$ be the set of segments used by at least one route, let $c_e$ be the number of nonempty routes containing segment $e$, and let $K_{\mathrm{eff}}$ be the number of routes with at least one segment. Each route contributes a segment at most once, so duplicate traversals within one route are not double-counted. We define
\begin{equation}
    \omega(\Pi)=
    \begin{cases}
    \dfrac{1}{|\mathcal{E}_{\Pi}|}
    \sum_{e\in\mathcal{E}_{\Pi}}
    \dfrac{c_e-1}{K_{\mathrm{eff}}-1},
    & K_{\mathrm{eff}}>1,\\[8pt]
    0, & K_{\mathrm{eff}}\le 1.
    \end{cases}
\end{equation}
Thus $\omega=0$ when no road segment is shared across routes and $\omega=1$ when every used segment appears in every nonempty route. This is an edge-count-based measure, not a length-weighted or node-overlap measure.

Table~\ref{tab:modal_split_results_extended} gives the full metric-level comparison behind the main-text results. \systemname{} achieves the strongest service-rate outcome at both modal splits, reaching $54.64\%$ at $\alpha=0.3$ and $82.08\%$ at $\alpha=1.0$. The $\alpha=0.3$ setting highlights the trade-off structure: \systemname{} has the highest service rate and bus utilization, while Neural Evolutionary has lower wait times and journey times at lower served demand and Shortest Path uses the smallest fleet. At $\alpha=1.0$, \systemname{} is stronger across most passenger and operator metrics, with the best service rate, wait time, route efficiency, and bus utilization. Relative to End-to-End RL, its service rate is $9.9\%$ and $11.4\%$ higher across the two regimes; relative to Pure MCTS, it is $2.5\%$ and $11.2\%$ higher.

Fig.~\ref{fig:learning_overview_1_0} shows that the same qualitative trends hold under the $\alpha=1.0$ objective. End-to-End RL benefits from dense reward shaping, with the $\Delta$-coverage variants giving the strongest PPO learning curves, but \systemname{} still achieves higher reward under the same environment-step budget. The search-time panel further shows that the learned policy/value network reduces MCTS decision cost by roughly two orders of magnitude relative to Pure MCTS at the same search depth, making search-guided route construction practical.

\begin{figure}[t!]
    \centering
    \includegraphics[width=0.98\linewidth]{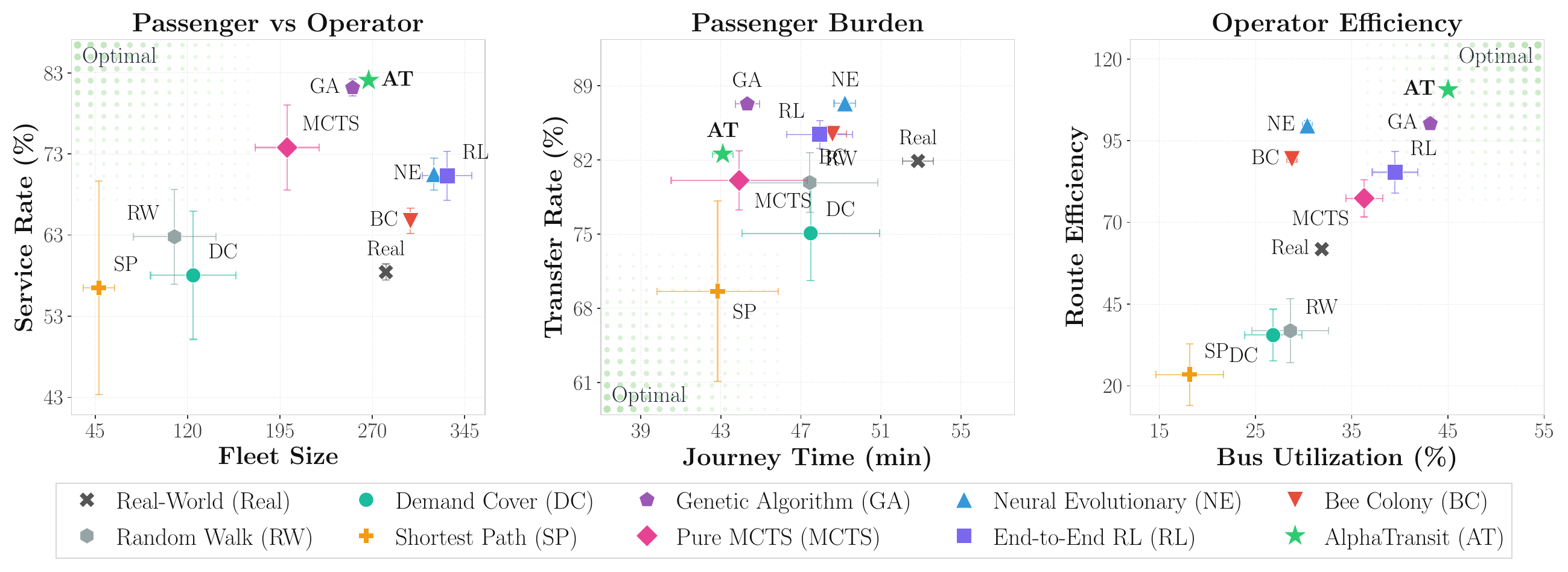}
    \vspace{-6pt}
     \caption{Full transit demand results on the Bloomington benchmark ($\alpha=1.0$). Points show means; error bars denote $\pm 1$ standard deviation. In each panel, green overlay and Optimal label mark the direction of improvement: upper-left for service rate versus fleet size, lower-left for journey time versus transfer rate, and upper-right for bus utilization versus route efficiency. \textbf{LEFT}: \systemname{} achieves highest service rate, $82.08\%$, with fleet size $267$. \textbf{MIDDLE}: \systemname{} obtains $43.10$ minutes of journey time and an $82.55\%$ transfer rate. \textbf{RIGHT}: \systemname{} achieves highest bus utilization, $45.02\%$, and highest route efficiency, $110.58$. Together, the panels show that \systemname{} pairs the highest service rate and operator efficiency with a competitive passenger-burden trade-off.}
    \label{fig:comparison_1_0}
    \vspace{-12pt}
\end{figure}

\begin{figure}[t!]
    \centering
    \vspace{10pt}
    \includegraphics[width=0.98\linewidth]{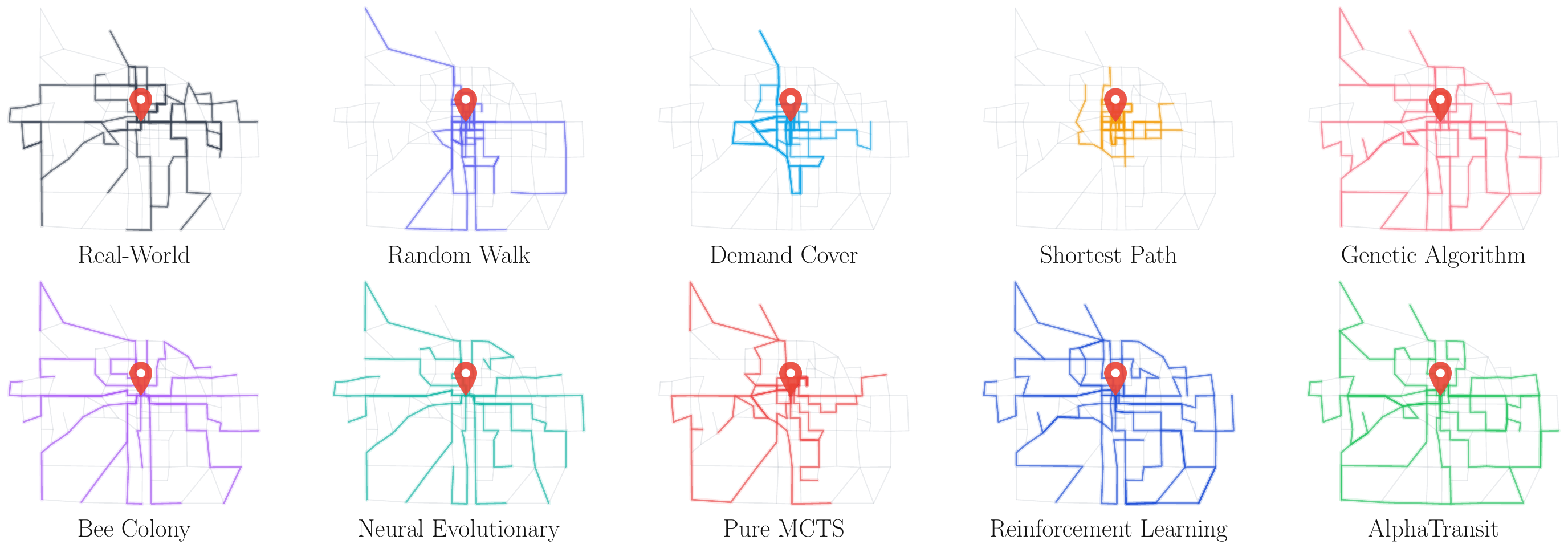}
    \vspace{-2pt}
    \caption{Selected route designs under full transit demand ($\alpha=1.0$) on the Bloomington street network. Each panel overlays selected routes on the same street basemap and marks the transit center. Pure MCTS denotes search without a learned policy/value network, while \systemname{} uses neural network-guided MCTS. Structural-analysis values show \systemname{} serves $128$ nodes, corresponding to $89.5\%$ node coverage, with $20.5\%$ shared-edge overlap and route distance $129.0$ km. In comparison, Real-World routes serve $114$ nodes with $79.7\%$ node coverage and $34.8\%$ shared-edge overlap, while End-to-End RL serves $135$ nodes with $94.4\%$ node coverage and $19.3\%$ shared-edge overlap i.e., \systemname{} is not simply maximizing node coverage relative to End-to-End RL.}
    \label{fig:routes_1_0}
    \vspace{-12pt}
\end{figure}

%% file: appendix_sections/broader_impacts.tex
\section{Broader Impacts}
\label{app:broader_impacts}

Urban population is projected to reach $60\%$ of the global population by $2030$~\cite{unwup2018}, placing pressure on transit infrastructure supporting $7.7$ billion annual trips in the United States alone~\cite{apta2025ridership}. In many large cities, less than $20\%$ of jobs can be reached within an hour using affordable transit, disproportionately affecting low-income workers~\cite{worldbank2024mobility}. \systemname{} is intended as decision support for transit agencies that must improve service quality under fleet, budget, and planning constraints, including frequency and wait time metrics central to ridership~\cite{transitcenter2019}. By evaluating candidate route networks with realistic topology, census-derived demand, and simulation-based outcomes, the framework can help agencies compare coverage, wait time, completed trips, and fleet use within a single evaluation loop. Potential harms could arise if optimized plans replace community input or if aggregate demand objectives underserve riders whose needs are not well captured in the data. Automated planning should therefore complement stakeholder engagement, and future work should incorporate fairness and accessibility constraints so algorithmic transit design supports inclusive mobility goals.

%% file: main.bib
@article{seo2025uxsim,
  title={UXsim: lightweight mesoscopic traffic flow simulator in pure Python},
  author={Seo, Toru},
  journal={Journal of Open Source Software},
  volume={10},
  number={106},
  pages={7617},
  year={2025}
}

@inproceedings{mumford2013new,
  title={New heuristic and evolutionary operators for the multi-objective urban transit routing problem},
  author={Mumford, Christine L},
  booktitle={2013 IEEE congress on evolutionary computation},
  pages={939--946},
  year={2013},
  organization={IEEE}
}

@article{huang2020closer,
  title={A closer look at invalid action masking in policy gradient algorithms},
  author={Huang, Shengyi and Onta{\~n}{\'o}n, Santiago},
  journal={arXiv preprint arXiv:2006.14171},
  year={2020}
}

@article{duran2022survey,
  title={A survey on the transit network design and frequency setting problem},
  author={Dur{\'a}n-Micco, Javier and Vansteenwegen, Pieter},
  journal={Public transport},
  volume={14},
  number={1},
  pages={155--190},
  year={2022},
  publisher={Springer}
}

@article{yoo2023reinforcement,
  title={A reinforcement learning approach for bus network design and frequency setting optimisation},
  author={Yoo, Sunhyung and Lee, Jinwoo Brian and Han, Hoon},
  journal={Public Transport},
  volume={15},
  number={2},
  pages={503--534},
  year={2023},
  publisher={Springer}
}

@misc{uscensus_lehd_lodes,
  author       = {{U.S. Census Bureau, Center for Economic Studies}},
  title        = {LEHD Origin-Destination Employment Statistics (LODES)},
  howpublished = {\url{https://lehd.ces.census.gov/data}},
  year         = 2025,
  note         = {Accessed: 2025-10-06}
}

@misc{uscensus_tiger_shapefiles,
  author       = {{U.S. Census Bureau}},
  title        = {TIGER/Line Shapefiles},
  howpublished = {\url{https://www.census.gov/cgi-bin/geo/shapefiles/index.php}},
  year         = 2025,
  note         = {Accessed: 2025-10-06}
}

@book{roess2011traffic,
  title={Traffic Engineering},
  author={Roess, R.P. and Prassas, E.S. and McShane, W.R.},
  isbn={9780136135739},
  lccn={2010013599},
  url={https://books.google.com/books?id=fGE6PgAACAAJ},
  year={2011},
  publisher={Pearson}
}

@techreport{mcguckin2018summary,
  title={Summary of travel trends: 2017 national household travel survey},
  author={McGuckin, Nancy and Fucci, Anthony and others},
  year={2018},
  institution={United States. Department of Transportation. Federal Highway Administration}
}

@misc{bloomingtontransit_gtfs,
  author       = {{Bloomington Transit}},
  title        = {GTFS Schedule Dataset},
  howpublished = {\url{https://bloomingtontransit.com/gtfs/}},
  institution  = {Bloomington Transit},
}

@misc{transitland_bloomingtontransit,
  author       = {{Transitland}},
  title        = {Bloomington Transit (BT) — Operator Details},
  howpublished = {\url{https://www.transit.land/operators/o-dnfq-bloomingtontransit}},
  institution  = {Transitland},
}

@inproceedings{kool2019attention,
  title={Attention, learn to solve routing problems!},
  author={Kool, Wouter and van Hoof, Herke and Welling, Max},
  booktitle={International Conference on Learning Representations},
  year={2019}
}

@article{brody2021attentive,
  title={How attentive are graph attention networks?},
  author={Brody, Shaked and Alon, Uri and Yahav, Eran},
  journal={arXiv preprint arXiv:2105.14491},
  year={2021}
}

@techreport{unwup2018,
  author = {{United Nations Department of Economic and Social Affairs}},
  title = {World Urbanization Prospects: The 2018 Revision},
  institution = {United Nations},
  year = {2019},
  url = {https://population.un.org/wup/}
}

@article{schmidt2024planning,
  title={Planning and optimizing transit lines},
  author={Schmidt, Marie and Sch{\"o}bel, Anita},
  journal={arXiv preprint arXiv:2405.10074},
  year={2024}
}

@article{schulman2017proximal,
  title={Proximal policy optimization algorithms},
  author={Schulman, John and Wolski, Filip and Dhariwal, Prafulla and Radford, Alec and Klimov, Oleg},
  journal={arXiv preprint arXiv:1707.06347},
  year={2017}
}

@article{schulman2015high,
  title={High-dimensional continuous control using generalized advantage estimation},
  author={Schulman, John and Moritz, Philipp and Levine, Sergey and Jordan, Michael and Abbeel, Pieter},
  journal={arXiv preprint arXiv:1506.02438},
  year={2015}
}

@article{huang202237,
  title={The 37 implementation details of proximal policy optimization},
  author={Huang, Shengyi and Dossa, Rousslan Fernand Julien and Raffin, Antonin and Kanervisto, Anssi and Wang, Weixun},
  journal={The ICLR Blog Track 2023},
  year={2022}
}

@misc{rl_bag_of_tricks,
  author       = {Jeremiah Coholich},
  title        = {A Bag of Tricks for Deep Reinforcement Learning},
  year         = {2023}, 
  url          = {https://www.jeremiahcoholich.com/post/rl_bag_of_tricks/#observation-normalization-and-clipping},
}

@book{sutton2018reinforcement,
  title={Reinforcement learning: An introduction},
  author={Sutton, Richard S and Barto, Andrew G},
  year={2018},
  publisher={MIT press}
}

@inproceedings{yang2022graph,
  title={Graph pointer neural networks},
  author={Yang, Tianmeng and Wang, Yujing and Yue, Zhihan and Yang, Yaming and Tong, Yunhai and Bai, Jing},
  booktitle={Proceedings of the AAAI conference on artificial intelligence},
  volume={36},
  number={8},
  pages={8832--8839},
  year={2022}
}

@article{velivckovic2017graph,
  title={Graph attention networks},
  author={Veličković, Petar and Cucurull, Guillem and Casanova, Arantxa and Romero, Adriana and Lio, Pietro and Bengio, Yoshua},
  journal={arXiv preprint arXiv:1710.10903},
  year={2017}
}

@article{newell2002simplified,
  title={A simplified car-following theory: a lower order model},
  author={Newell, Gordon Frank},
  journal={Transportation Research Part B: Methodological},
  volume={36},
  number={3},
  pages={195--205},
  year={2002},
  publisher={Elsevier}
}

@book{vuchic2017urban,
  title={Urban transit: operations, planning, and economics},
  author={Vuchic, Vukan R},
  year={2017},
  publisher={John Wiley \& Sons}
}

@book{vuchic2007urban,
  title={Urban transit systems and technology},
  author={Vuchic, Vukan R},
  year={2007},
  publisher={John Wiley \& Sons}
}

@article{buehler2019verkehrsverbund,
  title={Verkehrsverbund: The evolution and spread of fully integrated regional public transport in Germany, Austria, and Switzerland},
  author={Buehler, Ralph and Pucher, John and D{\"u}mmler, Oliver},
  journal={International Journal of Sustainable Transportation},
  volume={13},
  number={1},
  pages={36--50},
  year={2019},
  publisher={Taylor \& Francis}
}

@book{ceder2016public,
  title={Public transit planning and operation: Modeling, practice and behavior},
  author={Ceder, Avishai},
  year={2016},
  publisher={CRC press}
}

@article{kepaptsoglou2009transit,
  title={Transit route network design problem},
  author={Kepaptsoglou, Konstantinos and Karlaftis, Matthew},
  journal={Journal of transportation engineering},
  volume={135},
  number={8},
  pages={491--505},
  year={2009},
  publisher={American Society of Civil Engineers}
}

@article{zhao2006simulated,
  title={Simulated annealing--genetic algorithm for transit network optimization},
  author={Zhao, Fang and Zeng, Xiaogang},
  journal={Journal of Computing in Civil Engineering},
  volume={20},
  number={1},
  pages={57--68},
  year={2006},
  publisher={American Society of Civil Engineers}
}

@article{nikolic2013transit,
  title={Transit network design by bee colony optimization},
  author={Nikoli{\'c}, Milo{\v{s}} and Teodorovi{\'c}, Du{\v{s}}an},
  journal={Expert Systems with Applications},
  volume={40},
  number={15},
  pages={5945--5955},
  year={2013},
  publisher={Elsevier}
}

@article{holliday2025learning,
  title={Learning heuristics for transit network design and improvement with deep reinforcement learning},
  author={Holliday, Andrew and El-Geneidy, Ahmed and Dudek, Gregory},
  journal={Transportmetrica B: Transport Dynamics},
  volume={13},
  number={1},
  pages={2561863},
  year={2025},
  publisher={Taylor \& Francis}
}

@inproceedings{darwish2020optimising,
  title={Optimising public bus transit networks using deep reinforcement learning},
  author={Darwish, Ahmed and Khalil, Momen and Badawi, Karim},
  booktitle={2020 IEEE 23rd International Conference on Intelligent Transportation Systems (ITSC)},
  pages={1--7},
  year={2020},
  organization={IEEE}
}

@article{vermeir2021exact,
  title={An exact solution approach for the bus line planning problem with integrated passenger routing},
  author={Vermeir, Evert and Engelen, Wouter and Philips, Johan and Vansteenwegen, Pieter},
  journal={Journal of Advanced Transportation},
  volume={2021},
  number={1},
  pages={6684795},
  year={2021},
  publisher={Wiley Online Library}
}

@article{mandl1980evaluation,
  title={Evaluation and optimization of urban public transportation networks},
  author={Mandl, Christoph E},
  journal={European Journal of Operational Research},
  volume={5},
  number={6},
  pages={396--404},
  year={1980},
  publisher={Elsevier}
}

@article{heyken2019adaptive,
  title={An adaptive scaled network for public transport route optimisation},
  author={Heyken Soares, Philipp and Mumford, Christine L and Amponsah, Kwabena and Mao, Yong},
  journal={Public Transport},
  volume={11},
  number={2},
  pages={379--412},
  year={2019},
  publisher={Springer}
}

@book{fan2004optimal,
  title={Optimal transit route network design problem: Algorithms, implementations, and numerical results},
  author={Fan, Wei},
  year={2004},
  publisher={The University of Texas at Austin}
}

@article{vinyals2015pointer,
  title={Pointer networks},
  author={Vinyals, Oriol and Fortunato, Meire and Jaitly, Navdeep},
  journal={Advances in neural information processing systems},
  volume={28},
  year={2015}
}

@article{furth1981setting,
  title={Setting frequencies on bus routes: Theory and practice},
  author={Furth, Peter G and Wilson, Nigel HM},
  journal={Transportation Research Record},
  volume={818},
  number={1981},
  pages={1--7},
  year={1981}
}

@article{bello2016neural,
  title={Neural combinatorial optimization with reinforcement learning},
  author={Bello, Irwan and Pham, Hieu and Le, Quoc V and Norouzi, Mohammad and Bengio, Samy},
  journal={arXiv preprint arXiv:1611.09940},
  year={2016}
}

@inproceedings{li2023transit,
  title={A transit network design and frequency setting model with graph neural network and deep reinforcement learning},
  author={Li, Junjun and Dong, Hao and Zhao, Xuedong and Tang, Hao and Yin, Aimin and Xue, Ruchen},
  booktitle={Sixth International Conference on Computer Information Science and Application Technology (CISAT 2023)},
  volume={12800},
  pages={128005Y},
  year={2023},
  organization={SPIE}
}

@inproceedings{holliday2023augmenting,
  title={Augmenting transit network design algorithms with deep learning},
  author={Holliday, Andrew and Dudek, Gregory},
  booktitle={2023 IEEE 26th International Conference on Intelligent Transportation Systems (ITSC)},
  pages={2343--2350},
  year={2023},
  organization={IEEE}
}

@inproceedings{holliday2024neural,
  title={A neural-evolutionary algorithm for autonomous transit network design},
  author={Holliday, Andrew and Dudek, Gregory},
  booktitle={2024 IEEE International Conference on Robotics and Automation (ICRA)},
  pages={4457--4464},
  year={2024},
  organization={IEEE}
}

@article{silver2016mastering,
  author = {Silver, David and Huang, Aja and Maddison, Chris J. and Guez, Arthur and Sifre, Laurent and van den Driessche, George and Schrittwieser, Julian and Antonoglou, Ioannis and Panneershelvam, Veda and Lanctot, Marc and Dieleman, Sander and Grewe, Dominik and Nham, John and Kalchbrenner, Nal and Sutskever, Ilya and Lillicrap, Timothy and Leach, Madeleine and Kavukcuoglu, Koray and Graepel, Thore and Hassabis, Demis},
  title = {Mastering the game of {Go} with deep neural networks and tree search},
  journal = {Nature},
  volume = {529},
  pages = {484--489},
  year = {2016}
}

@article{silver2017mastering,
  author = {Silver, David and Schrittwieser, Julian and Simonyan, Karen and Antonoglou, Ioannis and Huang, Aja and Guez, Arthur and Hubert, Thomas and Baker, Lucas and Lai, Matthew and Bolton, Adrian and Chen, Yutian and Lillicrap, Timothy and Hui, Fan and Sifre, Laurent and van den Driessche, George and Graepel, Thore and Hassabis, Demis},
  title = {Mastering the game of {Go} without human knowledge},
  journal = {Nature},
  volume = {550},
  pages = {354--359},
  year = {2017}
}

@article{rosin2011multi,
  title={Multi-armed bandits with episode context},
  author={Rosin, Christopher D},
  journal={Annals of Mathematics and Artificial Intelligence},
  volume={61},
  number={3},
  pages={203--230},
  year={2011},
  publisher={Springer}
}

@article{silver2018general,
  author  = {Silver, David and Hubert, Thomas and Schrittwieser, Julian and Antonoglou, Ioannis and Lai, Matthew and Guez, Arthur and Lanctot, Marc and Sifre, Laurent and Kumaran, Dharshan and Graepel, Thore and Lillicrap, Timothy and Simonyan, Karen and Hassabis, Demis},
  title   = {A general reinforcement learning algorithm that masters chess, shogi, and {Go} through self-play},
  journal = {Science},
  volume  = {362},
  number  = {6419},
  pages   = {1140--1144},
  year    = {2018}
}

@article{weng2020pareto,
  title={Pareto-optimal transit route planning with multi-objective monte-carlo tree search},
  author={Weng, Di and Chen, Ran and Zhang, Jianhui and Bao, Jie and Zheng, Yu and Wu, Yingcai},
  journal={IEEE Transactions on Intelligent Transportation Systems},
  volume={22},
  number={2},
  pages={1185--1195},
  year={2021},
  publisher={IEEE}
}

@article{alkilane2025metrozero,
  author  = {Alkilane, Khalid and Lee, Der-Horng},
  title   = {{MetroZero}: Deep reinforcement learning and {Monte Carlo} tree search for optimized metro network expansion},
  journal = {IEEE Transactions on Intelligent Transportation Systems},
  volume  = {26},
  number  = {1},
  pages   = {810--823},
  year    = {2025}
}

@article{darvariu2023planning,
  author  = {Darvariu, Victor-Alexandru and Hailes, Stephen and Musolesi, Mirco},
  title   = {Planning spatial networks with {Monte Carlo} tree search},
  journal = {Proceedings of the Royal Society A: Mathematical, Physical and Engineering Sciences},
  volume  = {479},
  number  = {2269},
  pages   = {20220383},
  year    = {2023}
}

@article{amiri2024surrogate,
  title={Surrogate Assisted Monte Carlo Tree Search in Combinatorial Optimization},
  author={Amiri, Saeid and Zehtabi, Parisa and Dervovic, Danial and Cashmore, Michael},
  journal={arXiv preprint arXiv:2403.09925},
  year={2024}
}

@article{owais2026transit,
  title={Transit network design problem: a half century of methodological research},
  author={Owais, Mahmoud},
  journal={Innovative Infrastructure Solutions},
  volume={11},
  number={1},
  pages={3},
  year={2026},
  publisher={Springer}
}

@article{schrittwieser2020mastering,
  title        = {Mastering {Atari}, {Go}, chess and shogi by planning with a learned model},
  author       = {Schrittwieser, Julian and Antonoglou, Ioannis and Hubert, Thomas and Simonyan, Karen and Sifre, Laurent and Schmitt, Simon and Guez, Arthur and Lockhart, Edward and Hassabis, Demis and Graepel, Thore and Lillicrap, Timothy and Silver, David},
  journal      = {Nature},
  year         = {2020},
  volume       = {588},
  number       = {7839},
  pages        = {604--609},
  doi          = {10.1038/s41586-020-03051-4},
  publisher    = {Nature Publishing Group}
}

@article{fan2006optimal,
  title={Optimal transit route network design problem with variable transit demand: genetic algorithm approach},
  author={Fan, Wei and Machemehl, Randy B},
  journal={Journal of transportation engineering},
  volume={132},
  number={1},
  pages={40--51},
  year={2006},
  publisher={American Society of Civil Engineers}
}

@article{nayeem2014transit,
  title={Transit network design by genetic algorithm with elitism},
  author={Nayeem, Muhammad Ali and Rahman, Md Khaledur and Rahman, M Sohel},
  journal={Transportation Research Part C: Emerging Technologies},
  volume={46},
  pages={30--45},
  year={2014},
  publisher={Elsevier}
}

@inproceedings{lopez2018microscopic,
  author    = {Lopez, Pablo Alvarez and Behrisch, Michael and Bieker-Walz, Laura and Erdmann, Jakob and Fl{\"o}tter{\"o}d, Yun-Pang and Hilbrich, Robert and L{\"u}cken, Leonhard and Rummel, Johannes and Wagner, Peter and Wie{\ss}ner, Evamarie},
  title     = {Microscopic Traffic Simulation Using {SUMO}},
  booktitle = {2018 21st International Conference on Intelligent Transportation Systems (ITSC)},
  year      = {2018},
  pages     = {2575--2582},
  publisher = {IEEE},
  month     = nov
}

@article{kemmerling2023beyond,
  title={Beyond games: a systematic review of neural Monte Carlo tree search applications},
  author={Kemmerling, Marco and L{\"u}tticke, Daniel and Schmitt, Robert H},
  journal={arXiv preprint arXiv:2303.08060},
  year={2023}
}

@inproceedings{xu2018jk,
  title={Representation Learning on Graphs with Jumping Knowledge Networks},
  author={Xu, Keyulu and Li, Chengtao and Tian, Yonglong and Sonobe, Tomohiro and Kawarabayashi, Ken-ichi and Jegelka, Stefanie},
  booktitle={International Conference on Machine Learning (ICML)},
  year={2018}
}

@article{kwon2020pomo,
  title={Pomo: Policy optimization with multiple optima for reinforcement learning},
  author={Kwon, Yeong-Dae and Choo, Jinho and Kim, Byoungjip and Yoon, Iljoo and Gwon, Youngjune and Min, Seungjai},
  journal={Advances in neural information processing systems},
  volume={33},
  pages={21188--21198},
  year={2020}
}

@techreport{apta2025ridership,
  title = {Public Transportation Ridership Update},
  author = {American Public Transportation Association},
  institution = {APTA},
  year = {2025},
  type = {Policy Brief}
}

@techreport{worldbank2024mobility,
  title = {Promoting Livable Cities by Investing in Urban Mobility},
  author = {{World Bank}},
  institution = {World Bank Group},
  year = {2024},
  type = {Results Brief}
}

@techreport{transitcenter2019,
  title = {Who's On Board 2019: How to Win Back America's Transit Riders},
  author = {{TransitCenter}},
  institution = {TransitCenter},
  year = {2019},
  address = {New York}
}

@article{bertsimas2021data,
  title={Data-driven transit network design at scale},
  author={Bertsimas, Dimitris and Ng, Yee Sian and Yan, Julia},
  journal={Operations Research},
  volume={69},
  number={4},
  pages={1118--1133},
  year={2021}
}

@article{mankowitz2023faster,
  title={Faster sorting algorithms discovered using deep reinforcement learning},
  author={Mankowitz, Daniel J and Michi, Andrea and Zhernov, Anton and Gelmi, Marco and Selvi, Marco and Paduraru, Cosmin and Leurent, Edouard and Iqbal, Shariq and Lespiau, Jean-Baptiste and Ahern, Alex and others},
  journal={Nature},
  volume={618},
  number={7964},
  pages={257--263},
  year={2023},
  publisher={Nature Publishing Group UK London}
}

@article{fawzi2022discovering,
  title={Discovering faster matrix multiplication algorithms with reinforcement learning},
  author={Fawzi, Alhussein and Balog, Matej and Huang, Aja and Hubert, Thomas and Romera-Paredes, Bernardino and Barekatain, Mohammadamin and Novikov, Alexander and R. Ruiz, Francisco J and Schrittwieser, Julian and Swirszcz, Grzegorz and others},
  journal={Nature},
  volume={610},
  number={7930},
  pages={47--53},
  year={2022},
  publisher={Nature Publishing Group UK London}
}

@article{manser2020designing,
  title={Designing a large-scale public transport network using agent-based microsimulation},
  author={Manser, Patrick and Becker, Henrik and H{\"o}rl, Sebastian and Axhausen, Kay W},
  journal={Transportation Research Part A: Policy and Practice},
  volume={137},
  pages={1--15},
  year={2020},
  publisher={Elsevier}
}

@article{nnene2023simulation,
  title={A simulation-based optimization approach for designing transit networks},
  author={Nnene, Obiora A and Joubert, Johan W and Zuidgeest, Mark HP},
  journal={Public Transport},
  volume={15},
  number={2},
  pages={377--409},
  year={2023},
  publisher={Springer Verlag}
}

@techreport{cambridge2005congestion,
  author      = {{Cambridge Systematics, Inc.} and {Texas Transportation Institute}},
  title       = {Traffic Congestion and Reliability: {Trends} and Advanced Strategies for Congestion Mitigation},
  institution = {Federal Highway Administration},
  type        = {Final Report},
  number      = {FHWA-HOP-05-064},
  address     = {Washington, DC},
  year        = {2005},
  month       = sep,
  url         = {https://ops.fhwa.dot.gov/congestion_report/}
}

@inproceedings{coulom2006efficient,
  title={Efficient selectivity and backup operators in Monte-Carlo tree search},
  author={Coulom, R{\'e}mi},
  booktitle={International conference on computers and games},
  pages={72--83},
  year={2006},
  organization={Springer}
}

@article{browne2012survey,
  title={A survey of monte carlo tree search methods},
  author={Browne, Cameron B and Powley, Edward and Whitehouse, Daniel and Lucas, Simon M and Cowling, Peter I and Rohlfshagen, Philipp and Tavener, Stephen and Perez, Diego and Samothrakis, Spyridon and Colton, Simon},
  journal={IEEE Transactions on Computational Intelligence and AI in games},
  volume={4},
  number={1},
  pages={1--43},
  year={2012},
  publisher={IEEE}
}

@article{lawler1985traveling,
  title={The traveling salesman problem: a guided tour of combinatorial optimization},
  author={Lawler, Eugene L},
  journal={Wiley-Interscience Series in Discrete Mathematics},
  year={1985}
}

@book{toth2014vehicle,
  title={Vehicle routing: problems, methods, and applications},
  author={Toth, Paolo and Vigo, Daniele},
  year={2014},
  publisher={SIAM}
}
